%% file: arxiv_extra_material.tex
\documentclass[preprint]{article}

\input{packages}
\input{commands}

\usepackage[title]{appendix}

\title{Nonparametric estimation of continuous DPPs\\ with kernel methods}

\author{%
  Micha\"el Fanuel and R\'emi Bardenet\\
  Université de Lille, CNRS, Centrale Lille\\
  UMR 9189 – CRIStAL, F-59000 Lille, France\\
  \texttt{\{michael.fanuel, remi.bardenet\}@univ-lille.fr} \\
}

\begin{document}

\maketitle

\begin{abstract}
  Determinantal Point Process (DPPs) are statistical models for  repulsive point patterns.
  Both sampling and inference are tractable for DPPs, a rare feature among models with negative dependence that explains their popularity in machine learning and spatial statistics.
  Parametric and nonparametric inference methods have been proposed in the finite case, i.e. when the point patterns live in a finite ground set.
  In the continuous case, only parametric methods have been investigated, while nonparametric maximum likelihood for DPPs --~an optimization problem over trace-class operators~-- has remained an open question.
  In this paper, we show that a restricted version of this maximum likelihood (MLE) problem falls within the scope of a recent representer theorem for nonnegative functions in an RKHS.
  This leads to a finite-dimensional problem, with strong statistical ties to the original MLE.
  Moreover, we propose, analyze, and demonstrate a fixed point algorithm to solve this finite-dimensional problem. Finally, we also provide a controlled estimate of the correlation kernel of the DPP, thus providing more interpretability.
\end{abstract}

\section{Introduction}
Determinantal point processes (DPPs) are a tractable family of models for repulsive point patterns, where interaction between points is parametrized by a positive semi-definite kernel.
They were introduced by \cite{Mac75} in the context of fermionic optics, and have gained a lot of interest since the 2000s, in particular in probability \citep{HKPV06}, spatial statistics \citep{LaMoRu14}, and machine learning \citep{KuTa12}.
In machine learning at large, DPPs have been used essentially for two purposes: as statistical models for diverse subsets of items, like in recommendation systems \citep{Gartrell2019}, and as subsampling tools, like in experimental design \citep{pmlr-v108-derezinski20a}, column subset selection \citep{BeBaCh20b}, or Monte Carlo integration \citep{NEURIPS2019_1d54c76f,BeBaCh19}.
In this paper, we are concerned with DPPs used as statistical models for repulsion, and more specifically with inference for continuous DPPs.

DPP models in Machine Learning (ML) have so far mostly been \emph{finite} DPPs: they are distributions over subsets of a (large) finite ground set, like subsets of sentences from a large corpus of documents \citep{KuTa12}.
Since \citet{pmlr-v32-affandi14}, a lot of effort has been put into designing efficient inference procedures for finite DPPs.
In particular, the fixed point algorithm of \cite{pmlr-v37-mariet15} allows for nonparametric inference of a finite DPP kernel, thus learning the features used for modelling diversity from the data.
DPP models on infinite ground sets, say $\mathbb{R}^d$, while mathematically and algorithmically very similar to finite DPPs, have been less popular in ML than in spatial statistics.
It is thus natural that work on inference for \emph{continuous} DPPs has happened mostly in the latter community; see e.g. the seminal paper \citep{LaMoRu14}.
Inference for continuous DPPs has however focused on the parametric setting, where a handful of interpretable parameters are learned.
Relatedly, spatial statisticians typically learn the \emph{correlation} kernel of a DPP, which is more interpretable, while machine learners focus on the \emph{likelihood} kernel, with structural assumptions to make learning scale to large ground sets.

In this paper, we tackle \emph{nonparametric} inference for continuous DPPs using recent results on kernel methods.
More precisely, maximum likelihood estimation (MLE) for continuous DPPs is an optimization problem over trace-class operators.
Our first contribution is to show that a suitable modification of this problem is amenable to the representer theorem of \cite*{Marteau-Ferey}. Further drawing inspiration from the follow-up work \citep*{Rudi2020global}, we derive an optimization problem over matrices, and we prove that its solution has a near optimal objective in the original MLE problem.
We then propose, analyze, and demonstrate a fixed point algorithm for the resulting finite problem, in the spirit \citep{pmlr-v37-mariet15} of nonparametric inference for finite DPPs.
While our optimization pipeline focuses on the so-called likelihood kernel of a DPP, we also provide a controlled sampling approximation to its correlation kernel, thus providing more interpretability of our estimated kernel operator.
A by-product contribution of independent interest is an analysis of a sampling approximation for Fredholm determinants.

The rest of the paper is organized as follows.
Since the paper is notation-heavy, we first summarize our notation and give standard definitions in Section~\ref{s:notation}.
In Section~\ref{s:dpp}, we introduce DPPs and prior work on inference.
In Section~\ref{sec:MainResults}, we introduce our constrained maximum likelihood problem, and study its empirical counterpart.
We analyze an algorithm to solve the latter in Section~\ref{sec:Implementation}.
Statistical guarantees are stated in Section~\ref{sec:theory}, while Section~\ref{sec:simulations} is devoted to numerically validating the whole pipeline. Our code is freely available\footnote{\url{https://github.com/mrfanuel/LearningContinuousDPPs.jl}}.

\subsection{Notation and background\label{sec:def}}
\label{s:notation}
\paragraph{Sets.}
It is customary to define DPPs on a compact Polish space $\calX$ endowed with a Radon measure $\mu$, so that we can define the space of square integrable functions $L^2(\calX)$ for this measure~\citep{HKPV09}.
Outside of generalities in Section~\ref{s:dpp}, we consider a compact $\calX\subset \mathbb R^d$ and $\mu$ the uniform probability measure on $\calX$.
Let $(\calH, \langle \cdot, \cdot\rangle)$ be an RKHS of functions on $\calX$ with a bounded continuous kernel $\kRKHS(x,y)$, and let $\kappa^2 = \sup_{x\in \calX}\kRKHS(x,x)$.
Denote by $\phi(x) = \kRKHS(x,\cdot)\in \calH$ the canonical feature map.
For a Hilbert space $\mathcal F$, denote by $\Sb(\mathcal F)$ the space of symmetric and positive semi-definite trace-class operators on $\mathcal F$.
By a slight abuse of notation, we denote by $\Sb(\mathbb{R}^n)$ the space of $n\times n$ real positive semi-definite matrices. Finally, all sets are denoted by calligraphic letters (e.g.\ $\calC, \calI$).
\paragraph{Operators and matrices.} Trace-class endomorphisms of $L^2(\calX)$, seen as integral operators, are typeset as uppercase sans-serif (e.g.\ $\Lsf, \Kbb$), and the corresponding integral kernels as lowercase sans-serif (e.g.\ $\asf, \ksf$). Notice that $\ksf(x,y)$ and $\kRKHS(x,y)$ are distinct functions.
Other operators are written in standard fonts (e.g.\ $A, S$), while we write matrices and finite-dimensional vectors in bold (e.g.\ $\matK, \matC, \matv$).
The identity operator is written commonly as $\I$, whereas the $n\times n$ identity matrix is denoted by $\matI_n$.
When $\calC$ is a subset of $\{1,\dots,n\}$ and $\matK$ is an $n\times n$ matrix, the matrix $\matK_{\calC\calC}$ is the square submatrix obtained by selecting the rows and columns of  $\matK$ indexed by $\calC$.
\paragraph{Restriction and reconstruction operators.} Following~\citet[Section~3]{JMLR:v11:rosasco10a}, we define the restriction operator $\opS:\calH \to L^2(\calX)$ as
$(\opS g)(x) = g(x).$
Its adjoint $\opS^* : L^2(\calX)\to \calH$ is the reconstruction operator
$S^*  h = \int_\calX h(x) \phi(x) \rmd \mu(x).$
The classical integral operator  given by
$
\Lk h = \int_\calX \kRKHS(\cdot, x) h(x) \rmd \mu(x)
$, seen as an endomorphism of $L^2(\calX)$, thus takes the simple expression $\Lk = \opS\opS^* $.
Similarly, the so-called covariance operator
$\bbC : \calH \to \calH$, defined by
$
\bbC = \int_\calX \phi(x) \otimes \overline{\phi(x)} \rmd \mu(x),
$
writes  $\bbC = \opS^* \opS$.
In the tensor product notation defining $C$, $\overline{\phi(x)}$ is an element of the dual of $\calH$ and $\phi(x) \otimes \overline{\phi(x)}$ is the endomorphism of $\calH$ defined by $((\phi(x) \otimes \overline{\phi(x)})(g) = g(x)\phi(x)$; see e.g.~\cite{pmlr-v108-sterge20a}.
Finally, for convenience, given a finite set $\{x_1, \dots, x_n\}\subset \calX$, we also define \emph{discrete} restriction and reconstruction operators, respectively, as $S_n:\calH \to \R^n$  such that
$S_n g = (1/\sqrt{n})[g(x_1) , \dots , g(x_n)]^\top,$ and $S_n^* \matv = (1/\sqrt{n}) \sum_{i=1}^n \matv_i \phi(x_i)$ for any $\matv\in \R^n$. In particular,  we have $ \opS_n \opS_n^* = (1/n) \matK$ where  $\matK= [\kRKHS(x_i,x_j)]_{1\leq i,j \leq n}$ is a kernel matrix, which is defined for a given ordering of the set $\{x_1, \dots, x_n\}$. To avoid cumbersome expressions, when several discrete sets of different cardinalities, say $n$ and $p$, are used, we simply write the respective sampling operators as $S_n$ and $S_p$.

\section{Determinantal point processes and inference}
\label{s:dpp}

\paragraph{Determinantal point processes and L-ensembles.}
Consider a simple point process $\calY$ on $\calX$, that is, a random discrete subset of $\calX$.
For $\calD\subset \calX$, we denote by $\calY(\calD)$ the number of points of this process that fall within $\calD$.
Letting $m$ be a positive integer, we say that $\calX$ has $m$-point correlation function $\varrho_m$ w.r.t. to the reference measure $\mu$ if, for any mutually disjoint subsets $\calD_1, \dots,\calD_m\subset\calX$,
\[
  \E\left[\prod_{i=1}^m \calY(\calD_i)\right] = \int_{\prod_{i=1}^m \calD_i} \varrho_m(x_1,\dots, x_m) \rmd \mu(x_1)\dots \rmd \mu(x_m).
\]
In most cases, a point process is characterized by its correlation functions $(\rho_m)_{m\geq 1}$.
In particular, a determinantal point process (DPP) is defined as having correlation functions in the form of a determinant of a Gram matrix, i.e.  $\varrho_m(x_1,\dots,x_m) = \det[\ksf(x_i,x_j)]$ for all $m\geq 1$.
We then say that $\ksf$ is the \emph{correlation kernel} of the DPP.
Not all kernels yield a DPP:
if $\ksf(x,y)$ is the integral kernel of an operator $\Kbb\in \Sb(L^2(\calX))$, the Macchi-Soshnikov theorem \citep{Mac75,Sos00} states that the corresponding DPP exists if and only if the eigenvalues of $\Kbb$ are within  $[0,1]$.
In particular, for a finite ground set $\calX$, taking the reference measure to be the counting measure leads to conditions on the kernel matrix; see~\citet{KuTa12}.

A particular class of DPPs is formed by the so-called L-ensembles, for which the correlation kernel writes
\begin{equation}
\Kbb = \Lsf(\Lsf+ \I)^{-1},\label{eq:K_kernel_of_Lensemble}
\end{equation}
with the \emph{likelihood} operator $\Lsf\in \Sb(L^2(\calX))$ taken to be of the form $\Lsf f(x) = \int_\calX \asf(x,y)f(y)\rmd \mu(y)$.
The kernel $\asf$ of $\Lsf$ is sometimes called the \emph{likelihood kernel} of the L-ensemble, to distinguish it from its correlation kernel $\ksf$.
The interest of L-ensembles is that their Janossy densities can be computed in closed form.
Informally, the $m$-Janossy density describes the probability that the point process has cardinality $m$, and that the points are located around a given set of distinct points $x_1, \dots, x_m\in\calX$.
For the rest of the paper, we assume that $\calX\subset\mathbb{R}^d$ is compact, and that $\mu$ is the uniform probability measure on $\calX$; our results straightforwardly extend to other densities w.r.t. Lebesgue.
With these assumptions, the $m$-Janossy density is proportional to
\begin{align}
    \det(\mathbb{I}+ \Lsf)^{-1} \cdot \det[\asf(x_i,x_j)]_{1\leq i,j\leq m},
\end{align}
where the normalization constant is a Fredholm deteminant that implicitly depends on $\calX$.
Assume now that we are given $s$ i.i.d.\ samples of a DPP, denoted by the sets $\calC_1, \dots, \calC_s$ in the window $\calX$.
The associated maximum log-likelihood estimation problem reads
\begin{equation}
    \max_{\Lsf\in \Sb(L^2(\calX))} \frac{1}{s}\sum_{\ell=1}^{s}\log\det\left[ \asf(x_i,x_j) \right]_{i,j\in \calC_\ell}  -\log\det (\mathbb{I} + \Lsf).
    \label{eq:MLE_Problem}
\end{equation}
Solving~\eqref{eq:MLE_Problem} is a nontrivial problem. First, it is difficult to calculate the Fredholm determinant. Second, it is not straightforward to optimize over the space of operators $\Sb(L^2(X))$ in a nonparametric setting.
However, we shall see that the problem becomes tractable if we restrict the domain of~\eqref{eq:MLE_Problem} and impose regularity assumptions on the integral kernel $\asf(x,y)$ of the operator $\Lsf$.
For more details on DPPs, we refer the interested reader to~\citet{HKPV06,KuTa12,LaMoRu14}.

\paragraph{Previous work on learning DPPs.}
While continuous DPPs have been used in ML as sampling tools~\citep{BeBaCh19,BeBaCh20}
 or models~\citep{BaTi15,Ghosh13207},
  their systematic parametric estimation has been the work of spatial statisticians; see~\citet{LavancierMollerRubak, BiscioLavancier, poinas:hal-03157554} for general parametric estimation through \eqref{eq:MLE_Problem} or so-called \emph{minimum-contrast} inference.
Still for the parametric case, a two-step estimation was recently proposed for non-stationary processes by~\citet{lavancier:hal-01816528}. \MF{ In a more general context, non-asymptotic risk bounds for estimating a DPP density are given in~\citet{Baraud13}.}

Discrete DPPs have been more common in ML, and the study of their estimation has started some years ago~\citep{pmlr-v32-affandi14}.
Unlike continuous DPPs, nonparametric estimation procedures have been investigated for finite DPPs by~\cite{pmlr-v37-mariet15}, who proposed a fixed point algorithm.
Moreover, the statistical properties of maximum likelihood estimation of discrete L-ensembles were studied by~\cite{pmlr-v65-brunel17a}.
We can also cite low-rank approaches~\citep{pmlr-v84-dupuy18a,Gartrell2017},
learning with negative sampling \citep{pmlr-v89-mariet19b},
learning with moments and cycles \citep{pmlr-v70-urschel17a}, or learning with Wasserstein training \citep{anquetil2020wasserstein}.
Learning non-symmetric finite DPPs~\citep{Gartrell2019,gartrell2021scalable} has also been proposed, motivated by recommender systems.

At a high level, our paper is a continuous counterpart to the nonparametric learning of finite DPPs with symmetric kernels in \citet{pmlr-v37-mariet15}. Our treatment of the continuous case is made possible by recent advances in kernel methods.

\section{A sampling approximation to a constrained MLE \label{sec:MainResults}}
Using the machinery of kernel methods, we develop a controlled approximation of the MLE problem \eqref{eq:MLE_Problem}. Let us outline the main landmarks of our approach. First, we restrict the domain of the MLE problem~\eqref{eq:MLE_Problem} to smooth operators.
On the one hand, this restriction allows us to develop a sampling approximation of the Fredholm determinant.
On the other hand, the new optimization problem now admits a finite rank solution that can be obtained by solving a finite-dimensional problem. This procedure is described in Algorithm~\ref{alg:EstimateL} and yields an estimator for the likelihood kernel.
Finally, we use another sampling approximation and solve a linear system to estimate the correlation kernel of the fitted DPP; see Algorithm~\ref{alg:EstimateK}.

\paragraph{Restricting to smooth operators.}
In order to later apply the representer theorem of \cite{Marteau-Ferey}, we restrict the original maximum likelihood problem \eqref{eq:MLE_Problem} to ``smooth'' operators $\Lsf = SAS^*$, with $A\in\Sb(\calH)$ and $S$ the restriction operator introduced in Section~\ref{s:notation}.
Note that the kernel of $\Lsf$ now writes
\begin{equation}
    \asf(x,y) = \langle \phi(x), A \phi(y)\rangle.\label{eq:integral_kernel}
\end{equation}
With this restriction on its domain, the optimization problem \eqref{eq:MLE_Problem} now reads
\begin{equation}
    \min_{A\in \Sb(\calH)} f(A) = -\frac{1}{s}\sum_{\ell=1}^{s}\log\det\left[ \langle \phi(x_i), A \phi(x_j)\rangle \right]_{i,j\in \calC_\ell}  + \log\det (\mathbb{I} + SAS^* ).
    \label{eq:MLE_Problem_RKHS}
\end{equation}

\paragraph{Approximating the Fredholm determinant.} We use a sampling approach to approximate the normalization constant in~\eqref{eq:MLE_Problem_RKHS}.
We sample a set of points $\calI = \{x'_i: 1\leq i \leq n\}$ i.i.d.\ from the ambient probability measure $\mu$. \MF{For definiteness, we place ourselves on an event happening with probability one where all the points in $\calI$ and $\calC_\ell$ for $1\leq \ell \leq s$ are distinct.}
We define the sample version of $f(A)$ as
\[
  f_n(\opA) =  -\frac{1}{s}\sum_{\ell=1}^{s}\log\det\left[ \langle \phi(x_i), A \phi(x_j)\rangle \right]_{i,j\in \calC_\ell}  +\log\det (\matI_n + \opS_n \opA \opS_n^*),
\]
where the Fredholm determinant of $\Lsf = \opS \opA \opS^* $ has been replaced by the determinant of a \emph{finite} matrix involving $\opS_n \opA \opS_n^* = [\langle \phi(x'_i), \opL \phi(x'_j)\rangle]_{1\leq i,j \leq n}$.
\begin{theorem}[Approximation of the Fredholm determinant] \label{thm:formal_Fredholm}
  Let $\delta\in (0,1/2)$. With probability at least $1-2\delta$,
 \[
 \left|\log\det (\matI_n + \opS_n \opA \opS_n^*) - \log\det (\mathbb{I} + \opS \opA \opS^* )\right| \leq \log\det (\mathbb{I} + c_n \opA),
 \]
 with
 \[
 c_n = \frac{4\kappa^2 \log\left( \frac{2\kappa^2}{\ell \delta}\right)}{3n} + \sqrt{\frac{2\kappa^2 \ell \log\left( \frac{2\kappa^2}{\ell \delta}\right)}{n}},
 \]
 where $\ell = \lambda_{\max}(\Lk)$ and $\kappa^2 = \sup_{x\in \calX} \kRKHS(x,x) <\infty $.
 \end{theorem}
 The proof of Theorem~\ref{thm:formal_Fredholm} is given in Supplementary Material in Section~\ref{supp:thm:formal_Fredholm:proof}.
Several remarks are in order.
First, the high probability\footnote{We write $a\lesssim b$ if there exists a constant $c>0$ such that $a\leq c b$.} in the statement of Theorem~\ref{thm:formal_Fredholm} is that of the event
$
\{\|\opS^*\opS - \opS^*_n \opS_n\|_{op}\lesssim c_n\}.
$
Importantly, all the results given in what follows for the approximation of the solution of~\eqref{eq:MLE_Problem_RKHS} only depend on this event, so that we do not need any union bound. Second, we emphasize that $\Lk$, defined in Section~\ref{sec:def}, should not be confused with the correlation kernel~\eqref{eq:K_kernel_of_Lensemble}.
Third, to interpret the bound in Theorem~\ref{thm:formal_Fredholm}, it is useful to recall that $\log\det (\I +c_n \opA)\leq c_n \Tr(\opA)$, since $A\in \Sb(\calH)$.
Thus, by penalizing $\Tr(\opA)$, one also improves the upper bound on the Fredholm determinant approximation error.
This remark motivates the following infinite dimensional problem
\begin{equation}
        \min_{\opA\in \Sb(\calH)} f_n(\opA) + \lambda \Tr(\opA),
    \label{eq:MLE_Problem_Penalized}
\end{equation}
for some $\lambda >0$.
The penalty on $\Tr(A)$ is also intuitively needed so that the optimization problem selects a smooth solution, i.e., such a trace regularizer promotes a fast decay of eigenvalues of $A$.
Note that this problem depends both on the data $\calC_1, \dots, \calC_n\subset \calX$ and the subset $\calI$ used for approximating the Fredholm determinant.
\paragraph{Finite-dimensional representatives.} In an RHKS, there is a natural mapping between finite rank operators and matrices.
For the sake of completeness, let $\matK = [\kRKHS(z_i,z_j)]_{1\leq i,j \leq m}$ be a kernel matrix and let $\matK = \matR^\top \matR$ be a Cholesky factorization. Throughout the paper, kernel matrices are always assumed to be invertible. This is not a strong assumption: if $\kRKHS$ is the Laplace, Gaussian or Sobolev kernel, this is true almost surely if $z_i$ for $1\leq i \leq m$ are sampled e.g. w.r.t. the Lebesgue measure; see Bochner's classical theorem~\cite[Theorem 6.6 and Corollary 6.9]{Wendland_2004}. In this case, we can define a partial isometry $\opV:\calH \to \R^m$ as
$
\opV = \sqrt{m}(\matR^{-1})^\top S_{m}.
$
It satisfies $\opV\opV^* = \matI$, and $\opV^* \opV$ is the orthogonal projector onto the span of $ \phi(z_i)$ for all $1\leq i \leq m$. This is helpful to define
\begin{equation}
  \matPhi_i = \opV \phi(z_i)\in \R^n,\label{eq:Phi_i}
\end{equation}
the finite-dimensional representative of $\phi(z_i)\in \calH$ for all $1\leq i \leq m$. This construction yields a useful mapping between an operator in $\Sb(\calH)$ and a finite matrix, which is instrumental for obtaining our results.
\begin{lemma}[Finite dimensional representatives, extension of Lemma~3 in~\citet{Rudi2020global}]\label{lemma:Bbar}
Let $\opA\in\Sb(\calH)$. Then, the matrix $\bar{\matB} = \opV \opA \opV^*$  is such that
$
\matPhi_i^\top \bar{\matB} \matPhi_j = \langle \phi(z_i), \opL \phi(z_j)\rangle \quad \text{ for all } 1\leq i,j \leq m,
$
and $\log \det(\matI + \bar{\matB})\leq \log \det(\I + \opA)$, as well as $\Tr(\bar{\matB})\leq \Tr(\opA)$.
\end{lemma}
The proof of Lemma~\ref{lemma:Bbar} is given in Section~\ref{supp:lemma:Bbar:proof}.
Notice that the partial isometry $V$ also helps to map a matrix in $\Sb(\R^m)$ to an operator in $\Sb(\calH)$, as
$
  \matB \mapsto \opV^* \matB \opV,
$
in such a way that we have the matrix element matching $\langle \phi(z_i), \opV^* \matB \opV \phi(z_j)\rangle = \matPhi_i^\top \matB \matPhi_j$ for all $1\leq i,j \leq m$.
\paragraph{Finite rank solution thanks to a representer theorem.}
The sampling approximation of the Fredholm determinant also yields a finite rank solution for~\eqref{eq:MLE_Problem_Penalized}. For simplicity, we define $\calC \triangleq \cup_{\ell = 1 }^{s} \calC_\ell$ and recall $\calI = \{x'_1, \dots, x'_n\}$. Then, write the set of points $ \mathcal{Z}\triangleq\calC \cup \calI$ as $\{z_1,\dots, z_m\}$, with $m= |\calC| + n$, and denote the corresponding restriction operator $\opS_m:\calH\to \R^m$. Consider the kernel matrix $\matK = [\kRKHS(z_i,z_j)]_{1\leq i,j \leq m}$. In particular, since we used a trace regularizer,
the representer theorem of~\citet[Theorem 1]{Marteau-Ferey} holds: the optimal operator is of the form
\begin{equation}
  \opA = \sum_{i,j=1}^m \matC_{ij} \phi(z_i)\otimes \overline{\phi(z_j)} \text{ with } \matC\in \Sb(\R^m).\label{eq:A_representer}
\end{equation}
In this paper, we call $\matC$ the representer matrix of the operator $\opA$.
If we do the change of variables $\matB = \matR\matC\matR^\top$, we have the following identities:
$A = m S_m \matC S_m^* = V^* \matB V$ and $\Tr(A) = \Tr(\matK \matC) = \Tr(\matB)$,
thanks to Lemma~\ref{lem:eigenvalues} in Supplementary Material. Therefore, the problem~\eqref{eq:MLE_Problem_Penalized} boils down to the \emph{finite} non-convex problem:
\begin{equation}
    \min_{\matB\succeq 0}
f_{n}(\opV^* \matB \opV)+ \lambda  \Tr( \matB),\label{eq:SDP_B}
\end{equation}
where $f_{n}(\opV^* \matB \opV) = -\frac{1}{s}\sum_{\ell= 1}^{s}\log\det\left[\matPhi_i^\top \matB \matPhi_j  \right]_{i,j\in \calC_\ell}+\log\det\left[ \delta_{ij} +  \matPhi_i^\top \matB \matPhi_j /|\calI|  \right]_{i,j\in \calI}$. We assume that there is a global minimizer of~\eqref{eq:SDP_B} that we denote by $\matB_\star$.
The final estimator of the integral kernel of the likelihood $\Lsf$ depends on $\matC_\star = \matR^{-1}\matB_\star \matR^{-1\top}$ and reads
$
  \hat{\asf}(x,y) =  \sum_{i,j=1}^m \matC_{\star ij} \kRKHS(z_i,x)\kRKHS(z_j,y).
  $
The numerical strategy is summarized in Algorithm~\ref{alg:EstimateL}.
\begin{algorithm}[H]
  \begin{algorithmic}
    \Procedure{Estimate$\Lsf$}{$\lambda,\calC_1, \dots, \calC_s$} 
    \State Sample $\calI = \{x'_1, \dots, x'_n\}$ i.i.d.\ from $\mu$ \Comment{Sample $n$ points for Fredhom det. approx.}
    \State Define $\mathcal{Z} \triangleq  \cup_{\ell = 1 }^{s} \calC_\ell \cup \calI$ \Comment{Collect all samples}
    \State Compute $\matK = \matR^\top \matR$ with $\matK = [\kRKHS(z_i,z_j)]_{ i,j\in \mathcal{Z} }$ \Comment{Cholesky of kernel matrix}
    \State  Solve~\eqref{eq:SDP_B} with iteration~\eqref{eq:regPicard_with_B} to obtain $\matB_\star$ \Comment{Regularized Picard iteration}
    \State Compute $\matC_\star = \matR^{-1}\matB_\star\matR^{-1\top}$ \Comment{Representer matrix of $\hat{\asf}(x,y)$}
    \State {\bf return} $\hat{\asf}(x,y) =  \sum_{i,j=1}^m \matC_{\star ij} \kRKHS(z_i,x)\kRKHS(z_j,y) $ \Comment{Likelihood kernel}
  \EndProcedure
  \end{algorithmic}
  \caption{Estimation of the integral kernel $\asf(x,y)$ of the DPP likelihood kernel $\Lsf$.\label{alg:EstimateL}}
  \end{algorithm}
\paragraph{Estimation of the correlation kernel.}
The exact computation of the correlation kernel of the L-ensemble DPP
\begin{equation}
  \Kbb(\gamma) = \Lsf(\Lsf + \gamma \I)^{-1},\label{eq:K_L-Ensemble}
\end{equation}
requires the exact diagonalization of $\Lsf = \opS \opA\opS^*$. For more flexibility, we introduced a scale parameter $\gamma>0$ which  often takes the value $\gamma=1$. It is instructive to approximate $\Kbb$ in order to easily express the correlation functions of the estimated point process. We propose here an approximation scheme based once again on sampling. Recall the form of the solution $A =m \opS_m^* \matC \opS_m$ of~\eqref{eq:MLE_Problem_Penalized}, and consider the factorization $\matC = \matLambda^\top \matLambda$ with $\matLambda = \matF \matR^{-1\top}$ where $\matF^\top \matF = \matB_\star$ is the Cholesky factorization of $\matB_\star$. Let $\{x''_1,\dots, x''_p\}\subseteq \calX$ be sampled i.i.d.\ from the probability measure $\mu$ and denote by $\opS_p: \calH \to \R^p$ the corresponding restriction operator. The following integral operator
\begin{equation}
  \hat{\Kbb} =  m\opS \opS^*_m\matLambda^\top ( m\matLambda \opS_m \opS^*_p \opS_p \opS^*_m \matLambda^\top + \gamma \matI_m)^{-1}\matLambda \opS_m \opS^*,\label{eq:K_approx}
\end{equation}
%
%
gives an approximation of $\Kbb$.
The numerical approach for solving~\eqref{eq:K_approx} relies on the computation of $\matK_{m p} = \sqrt{mp} S_m S^*_p = [\kRKHS(z_i,x''_j)]$ with $1\leq i \leq m$ and $1\leq j \leq p$ is a rectangular kernel matrix, associated to a fixed ordering of $\mathcal{Z} = \{z_1, \dots, z_m \}$ and $\{x''_1,\dots, x''_p\}$. Our strategy is described in Algorithm~\ref{alg:EstimateK}.
\begin{algorithm}[H]
  \begin{algorithmic}
    \Procedure{Estimate$\Kbb$}{$\mathcal{Z},\matC_\star$} 
    \State Compute $\matC_\star = \matLambda^\top \matLambda$ \Comment{Factorization of representer matrix}
    \State Sample $\{x''_1,\dots, x''_p\}\subseteq \calX$ i.i.d.\ from $\mu$ \Comment{Sample $p$ points}
    \State Compute $\matK_{m p} = [\kRKHS(z_i,x''_j)]\in \R^{m\times p}$ \Comment{Cross kernel matrix}
    \State Compute $\matOmega=\matLambda^\top (\matLambda \matK_{m p}\frac{1}{p}\matK_{mp}^\top \matLambda^\top +  \matI_m)^{-1}\matLambda$ \Comment{Representer matrix of $\hat{\ksf} (x,y)$}
    \State {\bf return} $\hat{\ksf}(x,y) =  \sum_{i,j=1}^m \matOmega_{ij} \kRKHS(z_i,x)\kRKHS(z_j,y) $ \Comment{Correlation kernel}
  \EndProcedure
  \end{algorithmic}
  \caption{Estimation of the integral kernel $\ksf(x,y)$ of the DPP correlation kernel $\Kbb = \Lsf(\Lsf+ \I)^{-1}$. \label{alg:EstimateK}}
  \end{algorithm}

\section{Implementation}\label{sec:Implementation}
We propose an algorithm for solving the  discrete problem~\eqref{eq:SDP_B} associated to~\eqref{eq:MLE_Problem_Penalized}. To simplify the discussion and relate it to~\cite{pmlr-v37-mariet15}, we define the objective $g(\matX) = f_{n}(\opV^* \matB(\matX) \opV)+ \lambda \Tr( \matB(\matX))$ with the change of variables $\matB(\matX) = \matR^{-1 \top} \matX \matR^{-1}$. Then we can rephrase~\eqref{eq:SDP_B} as
\begin{equation}
  \min_{\matX\succeq 0}g(\matX) = -\frac{1}{s}\sum_{\ell= 1}^{s}\log\det(\matX_{\calC_\ell \calC_\ell})+\log\det\left( \matI_{|\calI|} + \frac{1}{n}\matX_{\calI\calI}  \right) + \lambda\Tr(  \matX \matK^{-1}),\label{eq:PicardObj}
\end{equation}
where we recall that $ n = |\calI|$.
Define for convenience $\matU_\ell$ as the matrix obtained by selecting the columns of the identity matrix which are indexed by $\calC_\ell$, so that, we have in particular $\matX_{\calC_\ell \calC_\ell} = \matU_\ell^\top \matX \matU_\ell$. Similarly, define a sampling matrix $\matU_\calI$ associated to the subset $\calI$. Recall the Cholesky decomposition $\matK = \matR^\top \matR$. To minimize~\eqref{eq:PicardObj}, we start at some $\matX_0\succ 0$ and use the following iteration
\begin{equation}
  \matX_{k+1} = \frac{1}{2\lambda}\matR^\top\left(\left(\matI_m +4\lambda \matR^{-1\top}p(\matX_k)\matR^{-1}\right)^{1/2}\matR -\matI_m\right)\matR, \label{eq:regPicard}
\end{equation}
where $p(\matX) = \matX+ \matX\matDelta \matX$  and $  \matDelta(\matX) = \frac{1}{s}\sum_{\ell= 1}^{s} \matU_\ell \matX_{\calC_\ell \calC_\ell}^{-1}\matU_\ell^\top -\matU_\calI(\matX_{\calI\calI}+ n \matI_{|\calI|})^{-1}\matU_\calI^\top.$
We dub this sequence a regularized Picard iteration, as it is a generalization of the Picard iteration which was introduced by~\cite{pmlr-v37-mariet15} in the context of learning discrete L-ensemble DPPs. In~\cite{pmlr-v37-mariet15}, the Picard iteration, defined as $\matX_{k+1} = p(\matX_k)$, is shown to be appropriate for minimizing a different objective  given by: $ -\frac{1}{s}\sum_{\ell= 1}^{s}\log\det(\matX_{\calC_\ell \calC_\ell})+\log\det( \matI + \matX )$.
The following theorem indicates that the iteration~\eqref{eq:regPicard} is a good candidate for minimizing $g(\matX)$.
\begin{theorem}\label{thm:Picard}
Let $\matX_k$ for integer $k$ be the sequence generated by~\eqref{eq:regPicard} and initialized with $\matX_0\succ 0$. Then, the sequence $g(\matX_k)$ is monotonically decreasing.
\end{theorem}
For a proof, we refer to Section~\ref{supp:thm:Picard:proof}.
In practice, we use the iteration~\eqref{eq:regPicard} with the inverse change of variables $\matX(\matB) = \matR^\top \matB \matR$ and solve
\begin{equation}
  \matB_{k+1} = \frac{1}{2\lambda}\left(\left( \matI_m + 4\lambda q(\matB_k)\right)^{1/2}-\matI_m\right), \text{ with } q(\matB) = \matB + \matB \matR \matDelta\big(\matX(\matB)\big) \matR^\top \matB,\label{eq:regPicard_with_B}
\end{equation}
where $\matDelta(\matX)$ is given hereabove. For the stopping criterion, we monitor the objective values of~\eqref{eq:SDP_B} and stop if the relative variation of two consecutive objectives is less than a predefined precision threshold $\mathsf{tol}$. Contrary to~\eqref{eq:PicardObj}, the objective~\eqref{eq:SDP_B} does not include the inverse of $\matK$, which might be ill-conditioned. The interplay between $\lambda$ and $n$ is best understood by considering~\eqref{eq:SDP_B} with the change of variables $\matB' = \matB/n$, yielding the equivalent problem
\[
  \min_{\matB'\succeq 0} -\frac{1}{s}\sum_{\ell= 1}^{s}\log\det\left(\matPhi^\top \matB' \matPhi  \right)_{\calC_\ell\calC_\ell}+\log\det\left( \matI +  \matPhi^\top \matB' \matPhi  \right)_{ \calI\calI} + \lambda n \Tr(\matB'),
\]
where $\matPhi = \matR$ is a matrix whose $i$-th column is $\matPhi_i$ for $1\leq i\leq m$ as defined in~\eqref{eq:Phi_i}. Notice that, up to a $\nicefrac{1}{n}$ factor, $\matPhi^\top \matB' \matPhi$ is the in-sample Gram matrix of $\hat{\asf}(x,y)$ evaluated on the data set $\mathcal{Z}$; see Algorithm~\ref{alg:EstimateL}.
Thus, in the limit $\lambda \to 0$, the above expression corresponds to the MLE estimation of a finite DPP if $\cup_\ell \calC_\ell \subseteq \calI$. This is the intuitive connection with finite DPP: the continuous DPP is well approximated by a finite DPP if the ground set $\calI$ is a dense enough sampling within $\calX$.
%
\section{Theoretical guarantees\label{sec:theory}}
We now describe the guarantees coming with the approximations presented in the previous section.
\paragraph{Statistical guarantees for approximating the maximum likelihood problem}
Next, we give a statistical guarantee for the approximation of the log-likelihood by its sample version.
\begin{theorem}[Discrete optimal objective approximates full MLE objective]\label{thm:bound_objectives_whp}
Let $\matB_\star $ be the solution of~\eqref{eq:SDP_B}. Let $\opL_\star$ be the solution of \eqref{eq:MLE_Problem_RKHS}.
Let $\delta\in (0,1/2)$. If $\lambda \geq  2 c_n(\delta)$, then with probability at least $1-2\delta$, it holds that
 \[
| f(\opL_\star) - f_n(\opV^*\matB_\star  \opV)| \leq  \frac{3}{2} \lambda \Tr ( \opL_\star ),
 \]
 with $0<c_n\lesssim 1/\sqrt{n}$ given in Theorem~\ref{thm:formal_Fredholm}.
 \end{theorem}
The above result, proved in Section~\ref{supp:thm:bound_objectives_whp:proof}, shows that, with high probability, the optimal objective value of the discrete problem is not far from the optimal log-likelihood provided $n$ is large enough. As a simple consequence, the discrete solution also yields a finite rank operator $V\matB_\star  V^*$ whose likelihood is not far from the optimal likelihood $f(\opL_\star)$, as it can be shown by using a triangle inequality.
\begin{corollary}[Approximation of the full MLE optimizer by a finite rank operator]\label{corol:likelihood_approximation}
Under the assumptions of Theorem~\ref{thm:bound_objectives_whp}, if $\lambda \geq  2 c_n(\delta)$, with probability at least $1-2\delta$, it holds
 \[
| f(\opL_\star) - f(\opV^* \matB_\star  \opV)| \leq  3 \lambda \Tr ( \opL_\star )
 \]
 with $c_n\lesssim 1/\sqrt{n}$ given in Theorem~\ref{thm:formal_Fredholm}.
\end{corollary}
The proof of Corollary~\ref{corol:likelihood_approximation} is also provided in Section~\ref{supp:thm:bound_objectives_whp:proof}.
\paragraph{Approximation of the correlation kernel}
An important quantity for the control of the amount of points necessary to approximate well the correlation kernel is the so-called \emph{effective dimension}
$
d_{\rm eff}(\gamma) = \Tr\left(\Lsf(\Lsf+\gamma\I)^{-1}\right),
$
which is the expected sample size under the DPP with correlation kernel $\Kbb = \Lsf(\Lsf + \gamma \I)^{-1}$.
\begin{theorem}[Correlation kernel approximation]\label{thm:approx_correlation_kernel_simplified}
  Let $\delta \in (0,1)$ be a failure probability, let $\epsilon\in (0,1)$ be an acurracy parameter and let $\gamma>0$ be a scale factor.
  Let $\Kbb(\gamma)$ be the correlation kernel~\eqref{eq:K_L-Ensemble} defined with $\Lsf = S A S^*$.
  Consider $\hat{\Kbb} (\gamma)$ defined in~\eqref{eq:K_approx} with i.i.d. sampling of $p$ points in $\calX$ wrt $\mu$.
  If we take
  $
    p\geq \frac{8\kappa^2 \| A\|_{op}}{\gamma \epsilon^2}\log\left( \frac{4 d_{\rm eff}(\gamma)}{\delta \| \Kbb \|_{op}}\right),
  $
  then, with probability $1-\delta$,  the following multiplicative error bound holds
  $
    \frac{1}{1+\epsilon}\Kbb(\gamma)\preceq \hat{\Kbb} (\gamma) \preceq \frac{1}{1-\epsilon} \Kbb(\gamma).
  $
  \end{theorem}
The proof of Theorem~\ref{thm:approx_correlation_kernel_simplified}, given in Section~\ref{supp:thm:approx_correlation_kernel_simplified:proof}, mainly relies on a matrix Bernstein inequality. Let us make a few comments.
First, we can simply take $\gamma = 1$ in Theorem~\ref{thm:approx_correlation_kernel_simplified} to recover the common definition of the correlation kernel~\eqref{eq:K_kernel_of_Lensemble}.
Second, the presence of $d_{\rm eff}(\gamma)$ in the logarithm is welcome since it is the expected subset size of the L-ensemble.
Third, the quantity $\|A\|_{op}$ directly influences the typical sample size to get an accurate approximation.
A worst case bound is $\|A\|_{op} \leq \lambda_{\max}(\matC) \lambda_{\max}(\matK)$ with $\matK = [\kRKHS(z_i,z_j)]_{1\leq i,j \leq m}$ and where we used that $A = m S_m \matC S_m^*$ in the light of~\eqref{eq:A_representer}. Thus, the lower bound on $p$ may be large in practice. Although probably more costly, an approach inspired from approximate ridge leverage score sampling~\citep{rudi2018fast} is likely to allow lower $p$'s. We leave this to future work.
\section{Empirical evaluation}\label{sec:simulations}
We consider an L-ensemble with correlation kernel $\ksf(x,y) =\rho \exp(-\|x-y\|_2^2/\alpha^2)$ defined on $\R^d$  with $\alpha = 0.05$.
Following \cite{LaMoRu14}, this is a valid kernel if $\rho< (\sqrt{\pi}\alpha)^{-d}$. Note that the intensity, defined as $x\mapsto\ksf(x,x)$, is constant equal to $\rho$; we shall check that the fitted kernel recovers that property.
%
%
\begin{figure}
  \centering
  \includegraphics[scale = 0.3]{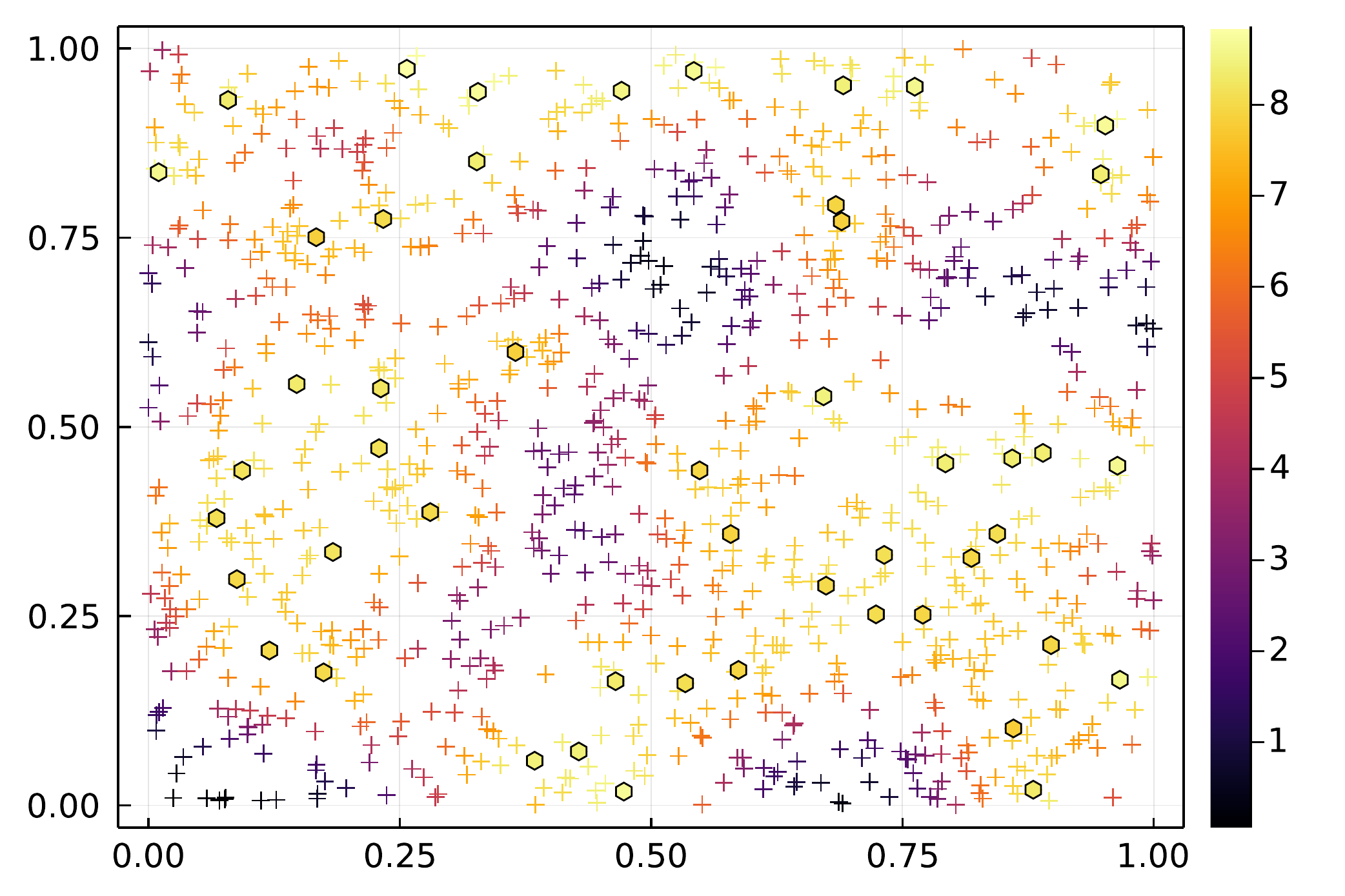}
  \includegraphics[scale = 0.3]{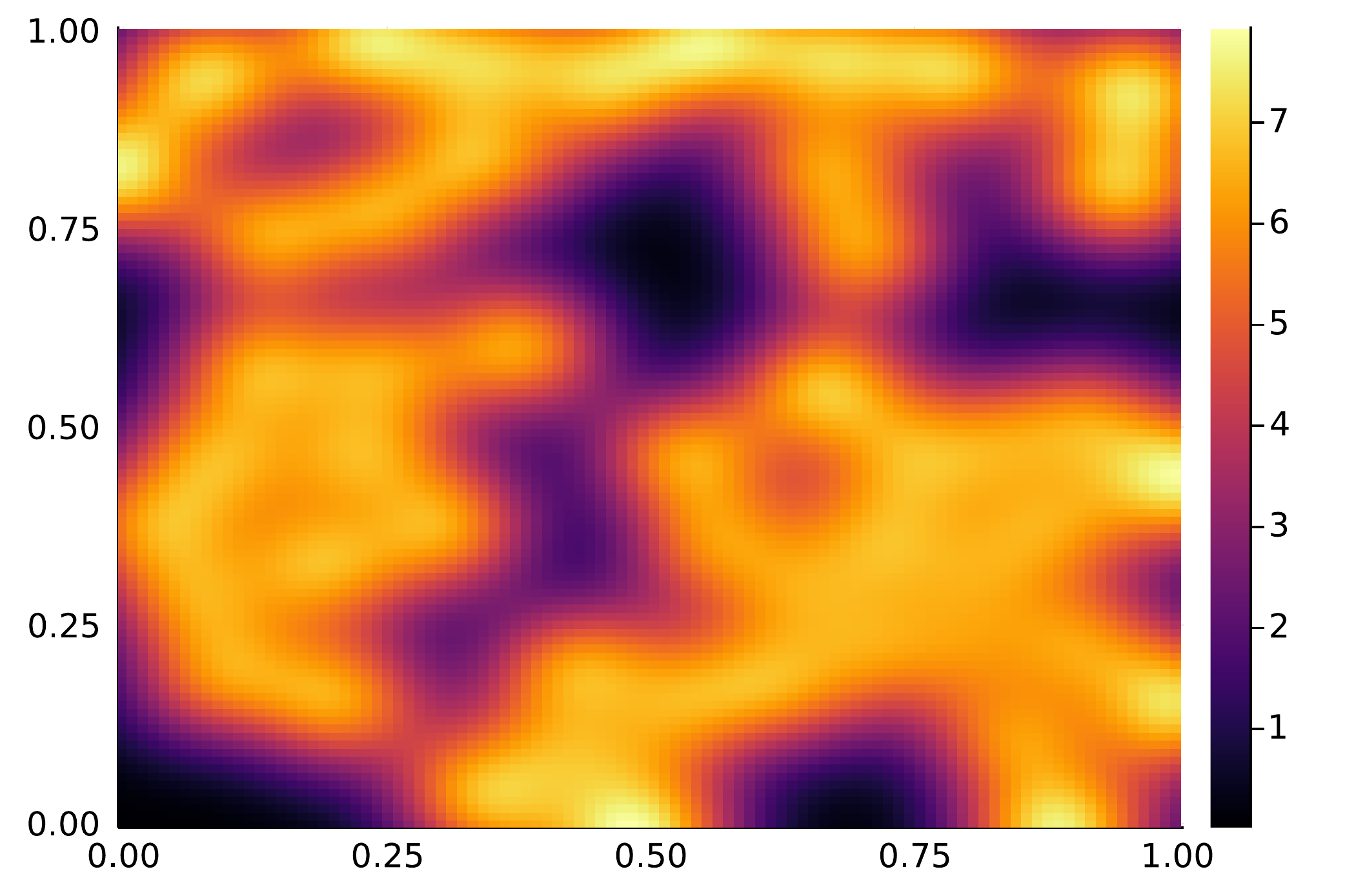}
  \includegraphics[scale = 0.3]{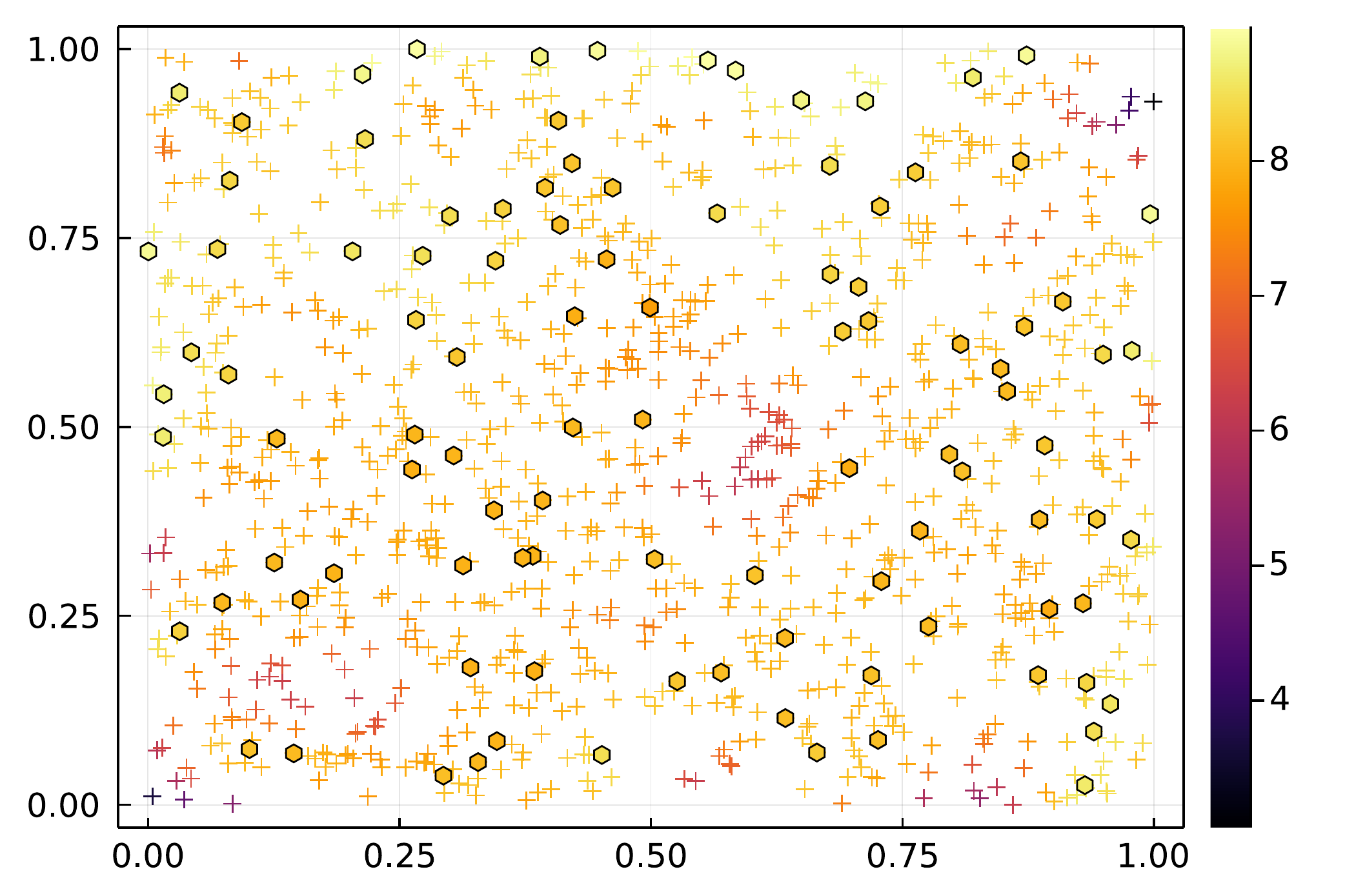}
  \includegraphics[scale = 0.3]{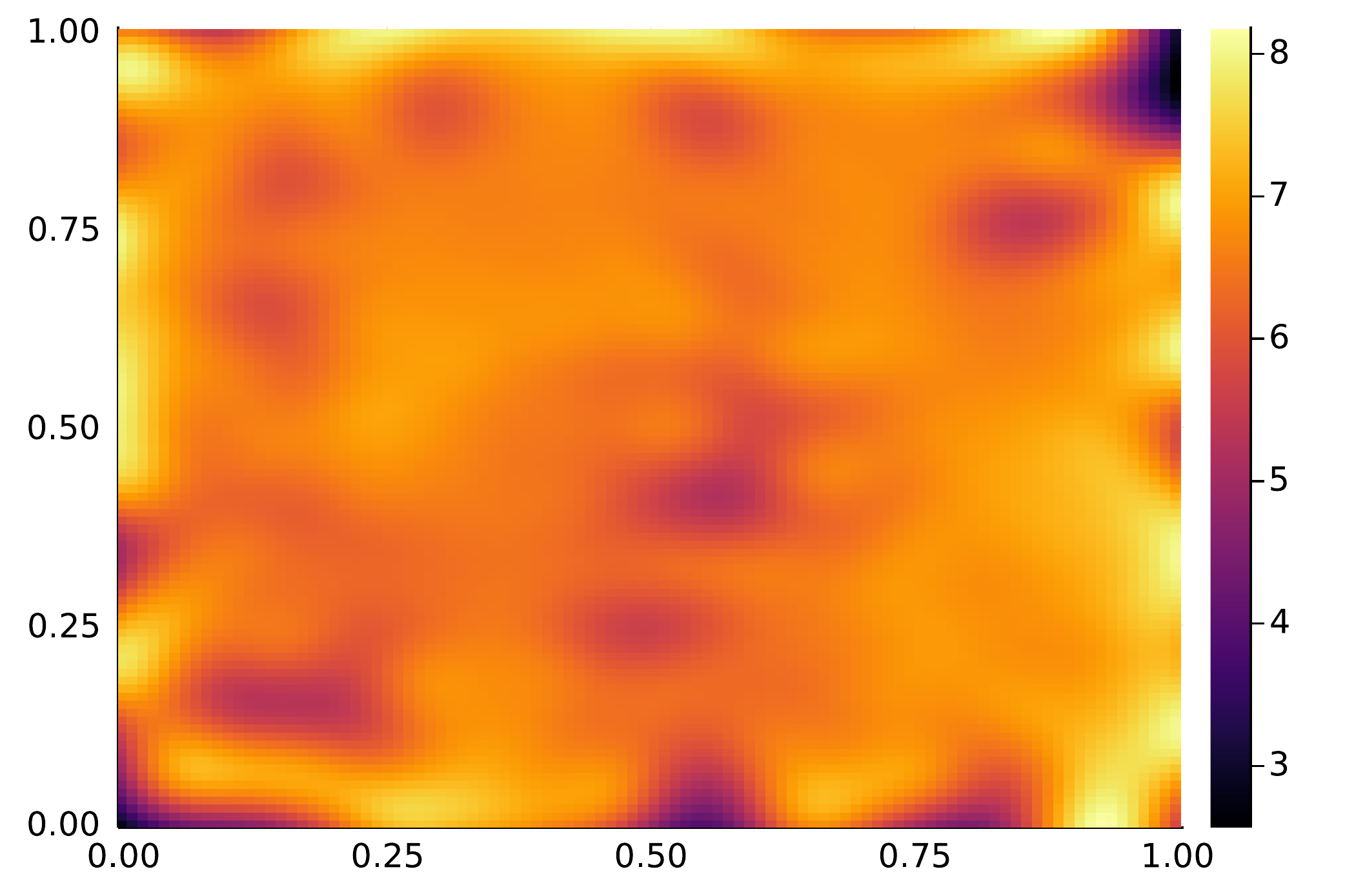}
  \caption{Intensity estimation with $\sigma = 0.1$ and $\lambda= 0.1$ from $1$ DPP sample with $\rho = 50$ (top row) and $\rho = 100$ (bottom row). On the LHS,  a DPP sample (hexagons) and $n = 1000$ uniform samples (crosses), the color is the diagonal of $\matPhi^\top \matB \matPhi$ (in-sample likelihood kernel). On the RHS, out-of-sample estimated intensity $\hat{\ksf}(x,x)$ of the learned process on a $100\times 100$ grid.
  \label{fig:Intensity1Sample}}
\end{figure}
We draw samples\footnote{We used the code of~\cite{poinas:hal-03157554}, available at \url{https://github.com/APoinas/MLEDPP}. It relies on the R package \emph{spatstat} \citep{Statspats}, available under the GPL-2 / GPL-3 licence.} from this continuous DPP in the window $\calX = [0,1]^2$.
Two such samples are shown as hexagons in the first column of Figure~\ref{fig:Intensity1Sample}, with respective intensity $\rho = 50$ and $\rho = 100$.
For the estimation, we use a Gaussian kernel $\kRKHS(x,y) = \exp\left(-\|x-y\|_2^2/(2\sigma^2)\right)$ with $\sigma>0$.
The computation of the correlation kernel always uses $p=1000$ uniform samples.
Iteration~\eqref{eq:regPicard_with_B} is run until the precision threshold $\mathsf{tol}= 10^{-5}$ is achieved. For stability, we add $10^{-10}$ to the diagonal of the Gram matrix $\matK$.
The remaining parameter values are given in captions.
We empirically observe that the regularized Picard iteration returns a matrix $\matB_\star$ such that $\matPhi^\top \matB_\star \matPhi$ is low rank; see Figure~\ref{fig:LowRank} (left).
A lesson from Figure~\ref{fig:Intensity1Sample} is that the sample size of the DPP has to be large enough to retrieve a constant intensity $\hat{k}(x,x)$. In particular, the top row of this figure illustrates a case where $\sigma$ is too small. Also, due to the large regularization $\lambda = 0.1$ and the use of only one DPP sample, the scale of $\rho$ is clearly underestimated in this example.
On the contrary, in Figure~\ref{fig:Intensity1SampleSmallLambda}, for a smaller regularization parameter $\lambda = 0.01$, the intensity scale estimate is larger.
We also observe that a large regularization parameter tends to smooth out the local variations of the intensity, which is not surprising.
\begin{figure}
  \centering
\includegraphics[scale = 0.2
]{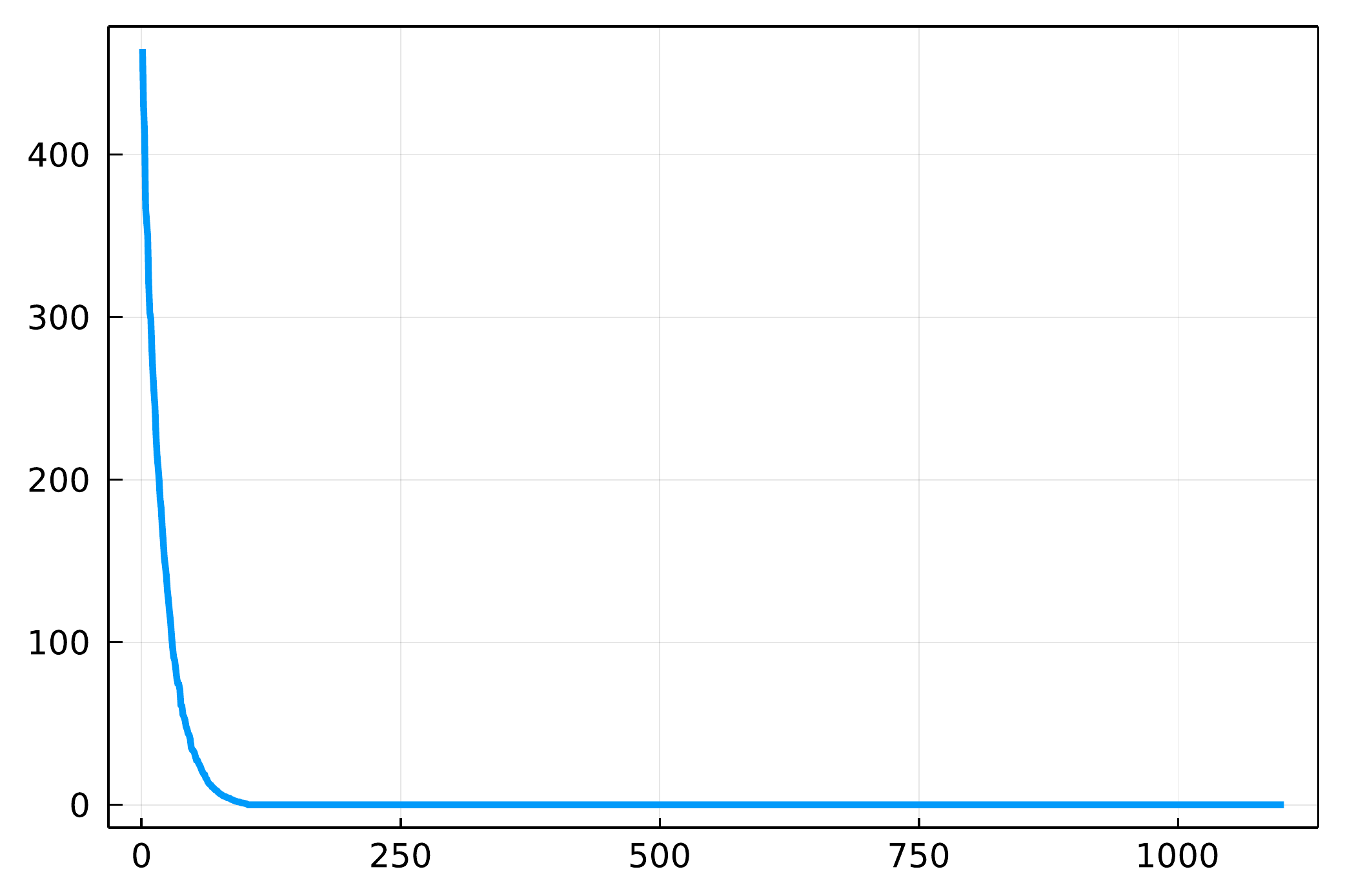}
\includegraphics[scale = 0.2
]{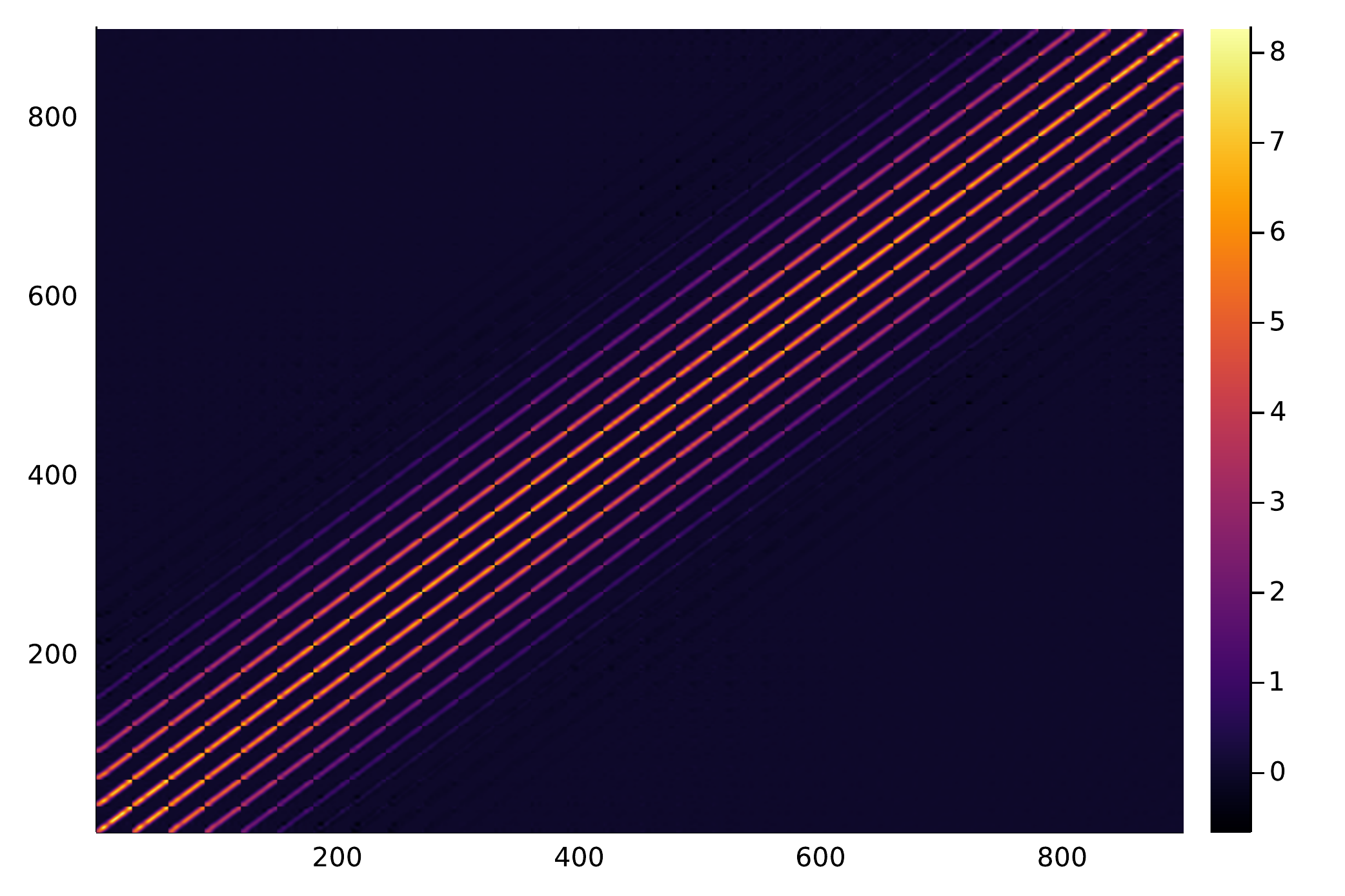}
\includegraphics[scale = 0.2
]{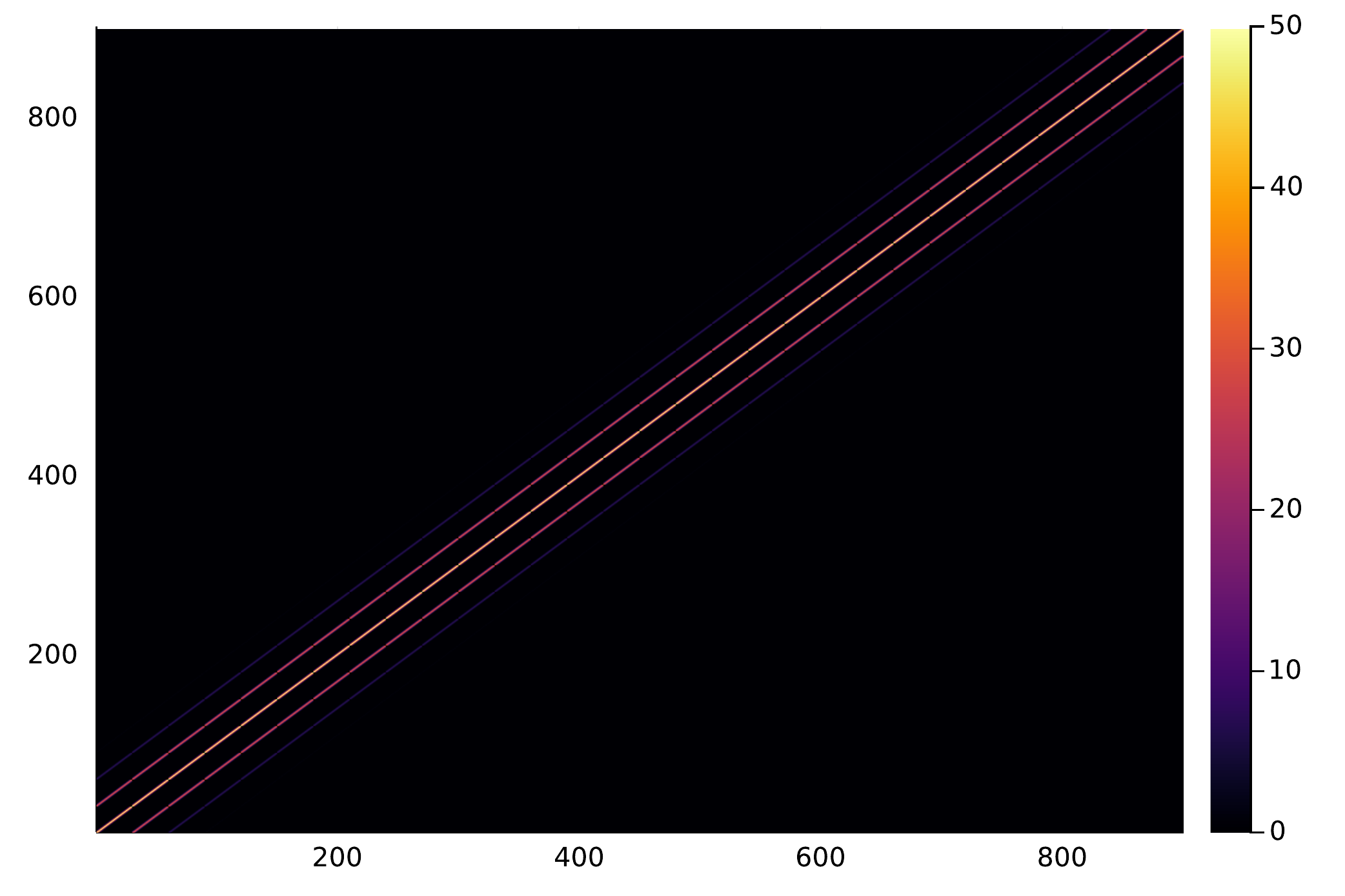}
\caption{Analysis of the solution corresponding to the example of Figure~\ref{fig:Intensity1Sample} with $\rho = 100$. Left: eigenvalues of $\matPhi^\top \matB\matPhi$. Middle: Gram matrix of $\hat{\ksf}(x,y)$ on a regular $30\times 30$ grid within $[0,1]^2$. Right: Gram matrix of $\ksf(x,y)$ on the same grid. \label{fig:LowRank}}
\end{figure}
A comparison between a Gram matrix of $\hat{\ksf}(x,y)$ and $\ksf(x,y)$ is given in Figure~\ref{fig:LowRank} corresponding to the example of Figure~\ref{fig:Intensity1Sample}. The short-range diagonal structure is recovered, while some long-range structures are smoothed out.
More illustrative simulations are given in Section~\ref{supp:OtherSimulations}, with a study of the influence of the hyperparameters, including the use of $s>1$ DPP samples. In particular, the estimation of the intensity is improved if several DPP samples are used with a smaller value of $\lambda$.
\begin{figure}
  \centering
  \includegraphics[scale = 0.3]{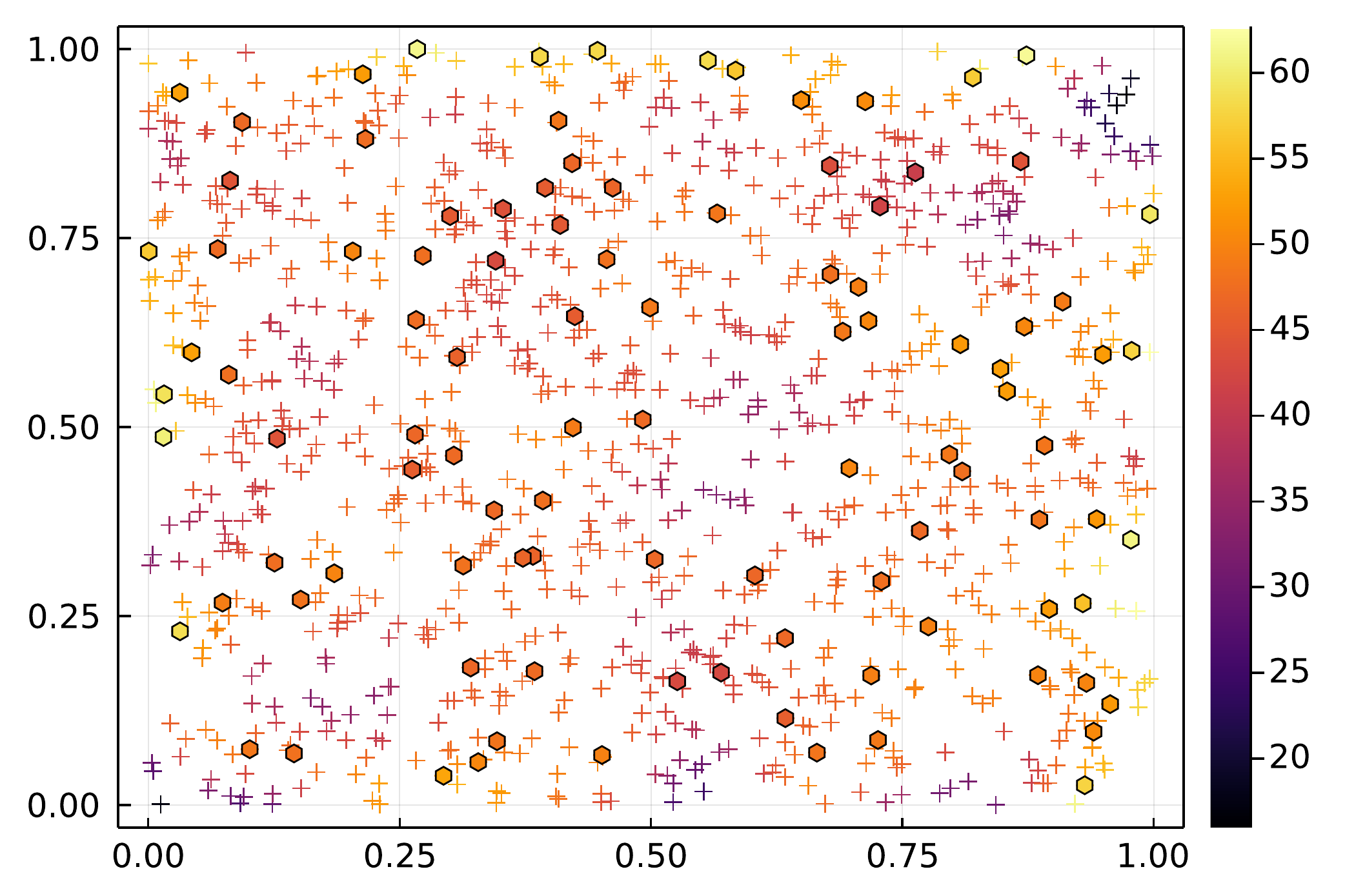}
  \includegraphics[scale = 0.3]{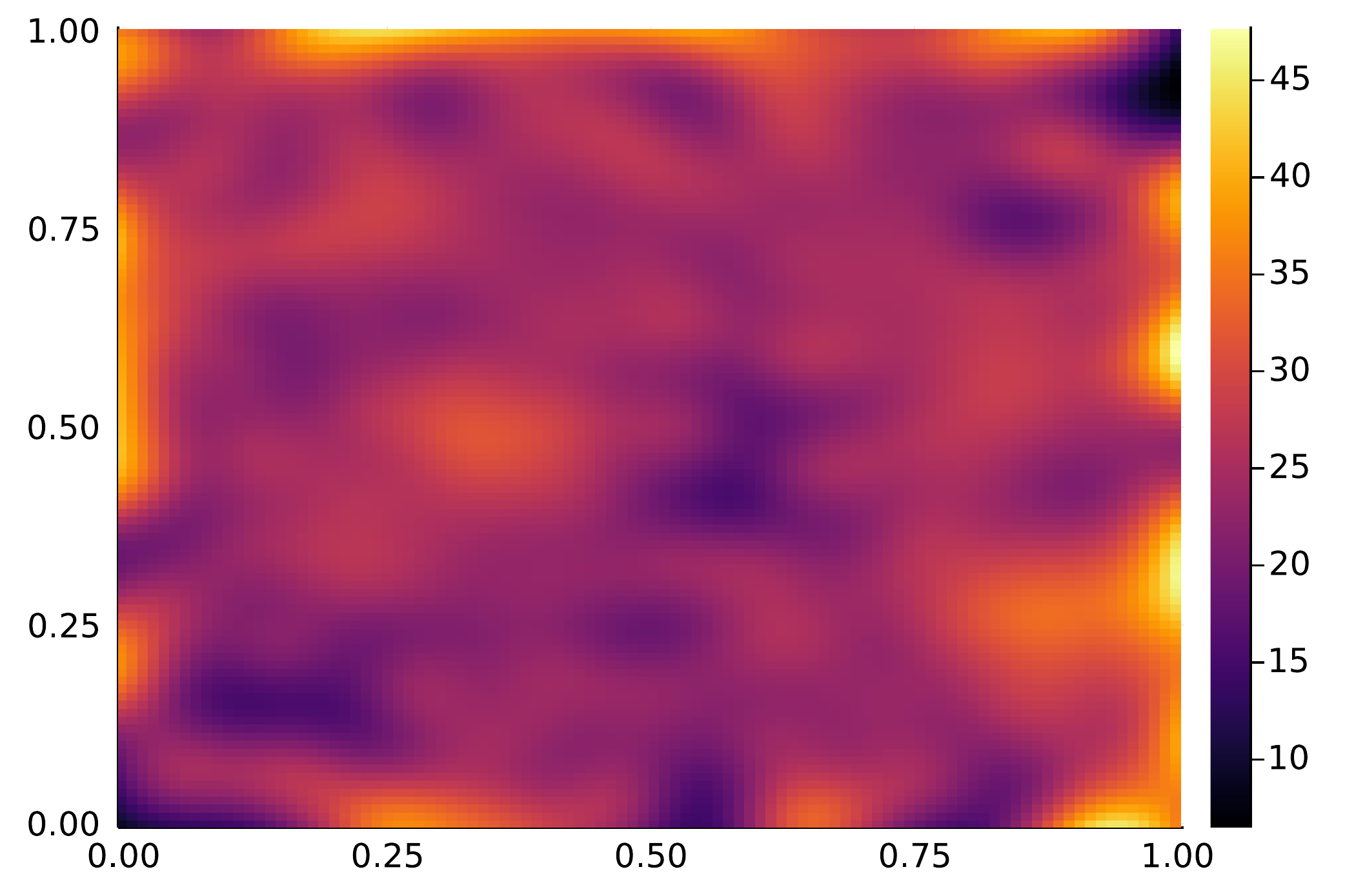}
  \caption{Intensity estimation with $\sigma = 0.1$ and $\lambda= 0.01$ from $1$ DPP sample with  $\rho = 100$. On the LHS, a DPP sample (hexagons) and $n = 1000$ uniform samples (crosses), the color is the diagonal of $\matPhi^\top \matB \matPhi$ (in-sample likelihood kernel). On the RHS, out-of-sample estimated intensity $\hat{\ksf}(x,x)$ of the learned process on a $100\times 100$ grid.
  \label{fig:Intensity1SampleSmallLambda}}
\end{figure}

\section{Discussion}
We leveraged recent progress on kernel methods to propose a nonparametric approach to learning continuous DPPs. We see three major limitations of our procedure.
First, our final objective function is nonconvex, and our algorithm is only guaranteed to increase its objective function. Experimental evidence suggests that our approach recovers the synthetic kernel, but more work is needed to study the maximizers of the likelihood, in the spirit of \citet{BMRU17Sub} for finite DPPs, and the properties of our fixed point algorithm.
Second, the estimated integral kernel does not have any explicit structure, other than being implicitly forced to be low-rank because of the trace penalty.
Adding structural assumptions might be desirable, either for modelling or learning purposes. For modelling, it is not uncommon to assume that the underlying continuous DPP is stationary, for example, which implies that the correlation kernel $\ksf(x,y)$ depends only on $x-y$.
For learning, structural assumptions on the kernel may reduce the computational cost, or reduce the number of maximizers of the likelihood.
The third limitation of our pipeline is that, like most nonparametric methods, it still requires to tune a handful of hyperparameters, and, in our experience, the final performance varies significantly with the lengthscale of the RKHS kernel or the coefficient of the trace penalty.
An automatic tuning procedure with guarantees would make the pipeline turn-key.

Maybe unexpectedly, future work could also deal with transferring our pipeline to the finite DPP setting.
Indeed, we saw in Section~\ref{sec:Implementation} that in some asymptotic regime, our regularized MLE objective is close to a regularized version of the MLE objective for a finite DPP.
Investigating this maximum a posteriori inference problem may shed some new light on nonparametric inference for finite DPPs.
Intuitively, regularization should improve learning and prediction when data is scarce.

\section*{Acknowledgements}
We thank Arnaud Poinas and Guillaume Gautier for useful discussions on the manuscript. We acknowledge support from ERC grant \textsc{Blackjack} (ERC-2019-STG-851866) and ANR AI chair \textsc{Baccarat} (ANR-20-CHIA-0002).

\appendix
\begin{center}
  \large\textbf{Supplementary Material}
\end{center}

\paragraph{Roadmap.} In Section~\ref{supp:sec:TechnicalLemmata}, we present useful technical results. Next, in a specific case, we show that the discrete problem~\eqref{eq:MLE_Problem_Penalized} admits a closed-form solution that we discuss in Section~\ref{supp:sec:ExplicitSolution}. Noticeably, this special case allows the understanding of the behaviour of the estimated DPP kernel in both the small and large regularization ($\lambda$) limits. In Section~\ref{supp:sec:Proofs}, all the deferred proofs are given.
Finally, Section~\ref{supp:OtherSimulations} provides a finer analysis of the empirical results of Section~\ref{sec:simulations} as well as a description of the convergence of the regularized Picard algorithm to the closed-form solution described in Section~\ref{supp:sec:ExplicitSolution}.

\section{Useful technical results \label{supp:sec:TechnicalLemmata}}
\subsection{Technical lemmata}
We preface the proofs of our main results with three lemmata.
\begin{lemma}\label{lem:eigenvalues}
Let $k$ be a strictly positive definite kernel of a RKHS $\calH$ and let $\matC$ be a $n\times n$ symmetric matrix. Assume that $\{x_i\}_{1\leq i\leq n}$ are such that the Gram matrix $\matK = [\kRKHS(x_i,x_j)]_{1\leq i,j\leq n}$  is non-singular.
Then, the non-zero eigenvalues of $\sum_{i,j=1}^n \matC_{ij} \phi(x_i)\otimes \overline{\phi(x_j)}$ correspond to the non-zero eigenvalues of  $\matK\matC$.
\end{lemma}
\begin{proof}
By definition of the sampling operator (see Section~\ref{sec:def}), it holds that
$
  S_n^* \matC S_n = \frac{1}{n}\sum_{i,j=1}^n \matC_{ij} \phi(x_i)\otimes \overline{\phi(x_j)}
$
and we have
$
S_n S_n^*\matC= \frac{1}{n}\matK\matC
$.
Thus, we need to show that the non-zero eigenvalues $S_n^* \matC S_n$ correspond to the non-zero eigenvalues of $S_n S_n^*\matC$.

Let $g_\lambda\in \calH$ be an eigenvector of $S_n^* \matC S_n$ with eigenvalue $\lambda\neq 0$.
First,  we show that $S_n g_\lambda$ is an eigenvector of $S_n S_n^*\matC$ with eigenvalue $\lambda$.
We have $S_n^* \matC S_n g_\lambda = \lambda g_\lambda$. By acting on both sides of the latter equation with $S_n$, we find
$S_n (S_n^* \matC S_n) g_\lambda = \lambda S_n g_\lambda$. This is equivalent to $S_n S_n^* \matC (S_n g_\lambda) = \lambda (S_n g_\lambda)$.
Second, since $S_n S_n^*\succ 0$, remark that $S_n S_n^*\matC$ is related by a similarity to $(S_n S_n^*)^{1/2}\matC(S_n S_n^*)^{1/2}$, which is diagonalizable. Since $S_n^* \matC S_n$ is at most of rank $n$, the non-zero eigenvalues of $S_n^* \matC S_n$ match the non-zero eigenvalues of $(S_n S_n^*)^{1/2}\matC(S_n S_n^*)^{1/2}$, which in turn are the same as the non-zero eigenvalues of $S_n S_n^*\matC$.
\end{proof}
\begin{lemma}\label{lem:concavity_h1} Let $\matSigma\in \Sb(\R^m)$ and let $\calI$ be a subset of $\{1,\dots,m\}$.
Then, the function
\begin{equation}
  \matSigma \mapsto   \log\det(\matSigma) + \log\det(\matI_m+\matSigma^{-1}/|\calI|)_{\calI\calI}
\label{e:function}
\end{equation}
is strictly concave on $\{\matSigma\succ 0\}$.
\end{lemma}
\begin{proof}
To simplify the expression, we do the change of variables $\matSigma\mapsto \matSigma/|\calI|$ and analyse $\log\det(\matSigma) + \log\det(\matI_m+\matSigma^{-1})_{\calI\calI}$ which differs from the original function by an additive constant.
Let $\matU_\calI$ be the matrix obtained by selecting the columns of the identity matrix which are indexed by $\calI$.
We rewrite the second term in \eqref{e:function} as
\[
  \log\det(\matI_m + \matSigma^{-1})_{\calI\calI} = \log\det(\matI_{|\calI|} + \matU_\calI^\top \matSigma^{-1}\matU_{\calI}) = \log\det(\matI_m+ \matSigma^{-1/2}\matU_{\calI}\matU_\calI^\top\matSigma^{-1/2}),
\]
by Sylvester's identity.
This leads to
\[
  \log\det(\matSigma) + \log\det(\matI_m+\matSigma^{-1})_{\calI\calI} = \log\det(\matSigma+ \matU_{\calI}\matU_\calI^\top).
\]
To check that $\log\det(\matSigma+ \matU_{\calI}\matU_\calI^\top)$ is strictly concave on $\{\matSigma\succ 0\}$, we verify that its Hessian is negative any direction $\matH$.
Its first order directional derivative reads $\Tr\left( (\matSigma+ \matU_{\calI}\matU_\calI^\top)^{-1}\matH\right)$.
Hence, the second order directional derivative in the direction of $\matH$ writes
\[
  -\Tr\left( (\matSigma+ \matU_{\calI}\matU_\calI^\top)^{-1} \matH (\matSigma+ \matU_{\calI}\matU_\calI^\top)^{-1} \matH\right),
\]
which is indeed a negative number.
\end{proof}
The following result is borrowed from~\citet[Lemma 2.3.]{pmlr-v37-mariet15}.
\begin{lemma}\label{lem:concavity_h2} Let $\matSigma\in \Sb(\R^m)$ and let $\matU\in \R^{m\times \ell}$ be a matrix with $\ell\leq m$ orthonormal columns. Then, the function
$
-\log\det(\matU^\top \matSigma^{-1} \matU)
$
is concave on $\matSigma\succ 0$.
\end{lemma}
\begin{proof}
The function $-\log\det(\matU^\top \matSigma^{-1} \matU)$ is concave on $\matSigma\succ 0$ since $\log\det(\matU^\top \matSigma^{-1} \matU)$ is convex on $\matSigma\succ 0$ for any $\matU$ such that $\matU^\top \matU = \matI$ as stated in~\citet[Lemma 2.3.]{pmlr-v37-mariet15}.
\end{proof}
\subsection{Use of the representer theorem}
We here clarify the definition of the representer theorem used in this paper.
\subsubsection{Extended representer theorem}

In Section~\ref{sec:MainResults}, we used a slight extension of the representer theorem of~\citet{Marteau-Ferey} which we clarify here.
Let us first define some notations. Let $\calH$ be a RKHS with feature map $\phi(\cdot)$ and $\{z_1, \dots, z_m\}$ be a data set such as defined in Section~\ref{sec:def}. Define 
\[
    h_A(z) = \langle \phi(z), A \phi(z)\rangle \text{ and } h_A(z,z') = \langle \phi(z), A \phi(z')\rangle.
\]
In this paper, we consider the problem
\begin{equation}
    \min_{A\in \Sb(\calH)} L\left(h_A(z_i,z_j)\right)_{1\leq i,j\leq m} + \Tr(A),\label{eq:extended_rep_problem}
\end{equation}
where $L$ is a loss function (specified below). In contrast, the first term of the problem considered in~\citet{Marteau-Ferey} is of the following form: $L(h_A(z_i))_{1\leq i\leq m}$. In other words, the latter loss function involves only diagonal elements $\langle \phi(z_i), A \phi(z_i)\rangle$ for $1\leq i\leq m$ while~\eqref{eq:extended_rep_problem} also involves off-diagonal elements.
Now, denote by $\Pi_m$ the projector on $\Span\{\phi(z_i), i=1,\dots, m\}$, and define
\[
    \mathcal{S}_{m,+}(\calH) = \{\Pi_m A \Pi_m : A \in \Sb(\calH)\}.
\]

Then, we have the following proposition.
\begin{proposition}[Extension of Proposition~7 in \cite{Marteau-Ferey}]
    Let $L$ be a lower semi-continuous function such that $L(h_A(z_i,z_j))_{1\leq i,j\leq m} + \Tr(A)$ is bounded below. Then~\eqref{eq:extended_rep_problem} has a solution $A_\star$ which is in $\mathcal{S}_{m,+}(\calH)$.
\end{proposition}
\begin{proof}[Proof sketch] The key step is the following identity
    \[
        h_A(z_i,z_j) = \langle \phi(z_i), A \phi(z_j)\rangle = \langle \phi(z_i), \Pi_m A \Pi_m \phi(z_j)\rangle = h_{\Pi_m A \Pi_m}(z_i,z_j), \ (1\leq i,j\leq m)
    \]
    which is a direct consequence of the definition of $\Pi_m$. Also, we have $\Tr(\Pi_m A \Pi_m)\leq \Tr(A)$. The remainder of the proof follows exactly the same lines as in~\cite{Marteau-Ferey}. Notice that, compared with~\citet[Proposition~7]{Marteau-Ferey}, we do not require the loss $L$ to be lower bounded but rather ask the full objective to be lower bounded, which is a weaker assumption but does not alter the argument.
\end{proof}
\subsubsection{Applying the extended representer theorem \label{sec:applying_representer}}
Now, let us prove that the objective of~\eqref{eq:MLE_Problem_Penalized}, given by 
$  
     f_n(\opA) + \lambda \Tr(\opA),
$
is lower bounded.
We recall that
\[
    f_n(\opA) =  -\frac{1}{s}\sum_{\ell=1}^{s}\log\det\left[ \langle \phi(x_i), A \phi(x_j)\rangle \right]_{i,j\in \calC_\ell}  +\log\det (\matI_n + \opS_n \opA \opS_n^*).
  \]
For clarity, we recall that $\calC \triangleq \cup_{\ell = 1 }^{s} \calC_\ell$ and  $\calI = \{x'_1, \dots, x'_n\}$. Then, write the set of points $ \mathcal{Z}\triangleq\calC \cup \calI$ as $\{z_1,\dots, z_m\}$ with $m= |\calC| + n$. Recall that we placed ourselves on an event happening with probability one where the sets $\calC_1, \dots, \calC_s, \calI$ are disjoint. Now, we denote by $\Pi_m$ the orthogonal projector on $\Span\{\phi(z_i): 1\leq i\leq m\}$, which writes
\begin{equation}
    \Pi_m = \sum_{i,j=1}^{m} (K^{-1})_{ij} \phi(z_i)\otimes \overline{\phi(z_j)} .\label{eq:projector_m}
\end{equation}

First, we have $\log\det (\matI_n + \opS_n \opA \opS_n^*)\geq 0$. Second, we use
\[
    \Tr(A) = \Tr\left(\Pi_m A\Pi_m\right)+ \Tr\left(\Pi_m^\perp A\Pi_m^\perp\right)\geq \Tr\left(\Pi_m A\Pi_m\right) = \Tr\left(\Pi_m A\right).
\]
Thus, a lower bound on the objective~\eqref{eq:MLE_Problem_Penalized} is
\begin{align*}
    f_n(\opA) + \lambda \Tr(\opA)\geq -\frac{1}{s}\sum_{\ell=1}^{s}\log\det\left[ \langle \phi(x_i), A \phi(x_j)\rangle \right]_{i,j\in \calC_\ell} + \lambda \Tr\left(\Pi_m A\right).
\end{align*}
Now we define the matrix $M$ with elements $M_{ij} = \langle \phi(z_i), A \phi(z_j)\rangle$ for $1\leq i,j\leq m$ and notice that \[\Tr\left(\Pi_m A\right) = \Tr(M K^{-1})\geq \Tr(M)/\lambda_{\max}(K).\]
Remark that the matrix defined by $M^{(\ell)} = \left[ \langle \phi(x_i), A \phi(x_j)\rangle \right]_{i,j\in \calC_\ell}$  associated with the $\ell$-th DPP sample for $1\leq \ell\leq s$ is a \emph{principal} submatrix of the $m\times m$ matrix $M = \left[ \langle \phi(z_i), A \phi(z_j)\rangle \right]_{1\leq i,j\leq m}$. Since the sets $\calC_\ell$ are disjoint, we have $\sum_{\ell=1}^{s}\Tr(M^{(\ell)}) \leq \Tr(M)$.
Denoting $\lambda' = \frac{\lambda}{\lambda_{\max}(K)}$, by using the last inequality, we find
\begin{align*}
    f_n(\opA) + \lambda \Tr(\opA) & \geq -\frac{1}{s}\sum_{\ell=1}^{s}\log\det M^{(\ell)} +  \lambda' \sum_{\ell=1}^{s}\Tr(M^{(\ell)})\\
    &= \frac{1}{s}\sum_{\ell=1}^{s}\left( - \log\det M^{(\ell)} + s \lambda'  \Tr(M^{(\ell)}) \right).
\end{align*}
Finally, we use the following proposition to show that each term in the above sum is bounded from below.
\begin{proposition}\label{prop:h_bounded} Let $a>0$. The function  $ h(\sigma) = -\log(\sigma) + a \sigma$ satisfies $h(\sigma) \geq h(1/a)= 1 + \log(a)$ for all $\sigma>0$.
\end{proposition}
Thus a lower bound for the objective of~\eqref{eq:MLE_Problem_Penalized} is obtained by applying Proposition~\ref{prop:h_bounded} with $a= s\lambda/\lambda_{\max}(K)$ to the each eigenvalue of $M^{(\ell)}$ for all $1\leq \ell \leq s$.
\subsection{Boundedness of the discrete objective function}
At first sight, we may wonder if the objective of the optimization problem~\eqref{eq:PicardObj} is lower bounded. We show here that the optimization problem is well-defined for all regularization parameters $\lambda>0$. \MF{The lower boundedness of the discrete objective follows directly from Section~\ref{sec:applying_representer} in the case of a finite rank operator. For completeness, we repeat below the argument in the discrete case.}
Explicitly, this discrete objective reads
\begin{equation*}
    g(\matX) = -\frac{1}{s}\sum_{\ell= 1}^{s}\log\det(\matX_{\calC_\ell \calC_\ell})+\log\det\left( \matI_{|\calI|} + \frac{1}{n}\matX_{\calI\calI}  \right) + \lambda\Tr(  \matX \matK^{-1}),
  \end{equation*}
for all $\matX\succeq 0$. First, since $\matX_{\calI\calI}\succeq 0$, we have $\log\det\left( \matI_{|\calI|} + \frac{1}{n}\matX_{\calI\calI}  \right)\geq 0$. Next, we use that $\matK\succ 0$ by assumption, which implies that $\Tr(\matX \matK^{-1})\geq \Tr(\matX)/\lambda_{\max}(\matK)$. Denote $\lambda' = \frac{\lambda}{\lambda_{\max}(\matK)}$. Thus, we can lower bound the objective function as follows:
\begin{align*}
    g(\matX)&\geq -\frac{1}{s}\sum_{\ell= 1}^{s}\log\det(\matX_{\calC_\ell \calC_\ell}) + \lambda'\Tr(  \matX)\\
    &\geq \frac{1}{s}\sum_{\ell= 1}^{s}\big(-\log\det(\matX_{\calC_\ell \calC_\ell}) + \lambda's\Tr(  \matX_{\calC_\ell \calC_\ell})\big),
\end{align*}
where we used that $\Tr(\matX)= \Tr(\matX_{\calI\calI}) + \sum_{\ell= 1}^{s}\Tr(\matX_{\calC_\ell \calC_\ell})$ with $\Tr(\matX_{\calI\calI})\geq 0$ for $\matX\succeq 0$.  Hence, by using Proposition~\ref{prop:h_bounded}, $g(\matX)$ is bounded from below by a sum of functions which are lower bounded.

\section{Extra result: an explicit solution for the single-sample case \label{supp:sec:ExplicitSolution}}
We note that if the dataset is made of only one sample (i.e., $s=1$), and if that sample is also used to approximate the Fredholm determinant (i.e., $\mathcal I=\mathcal C$), then the problem~\eqref{eq:MLE_Problem_Penalized} admits an explicit solution.
%
%
\begin{proposition}\label{prop:exact}
Let $\calC = \{x_i\in \calX\}_{1\leq i\leq m}$ be such that the kernel matrix $\matK= [\kRKHS(x_i,x_j)]_{1\leq i, j\leq m}$ is invertible.
Consider the problem~\eqref{eq:MLE_Problem_Penalized} with $\calI=\calC$ and $\lambda >0$.
Then the solution of~\eqref{eq:MLE_Problem_Penalized} has the form $\opL_\star = \sum_{i,j=1}^m \matC_{\star ij} \phi(x_i)\otimes \overline{\phi(x_j)}$,
with
$$
\matC_\star = \matC_\star (\lambda)=  \frac{1}{2}\matK^{-2} \left((m^2\matI_m + 4m\matK /\lambda)^{1/2}-m\matI_m\right).
$$
\end{proposition}
\begin{proof}
  By the representer theorem in \citet[Theorem 1]{Marteau-Ferey}, the solution is of the form $A = \sum_{i,j=1}^m \matC_{ij} \phi(x_i)\otimes \overline{\phi(x_j)}$ where $\matC$ is the solution of
  \[
  \min_{\matC\succ 0}
  -\log\det\left( \matK \matC \matK  \right)+\log\det\left( \matI_m + \matK \matC \matK/m  \right) + \lambda  \Tr( \matK \matC),
  \]
  Define $\matX = \matK\matC\matK$ and denote the objective by $g(\matX) = -\log\det(\matX)+\log\det( \matI_m + \frac{1}{m} \matX) + \lambda \Tr( \matK^{-1} \matX)$.
  Let $\matH$ be a $m\times m$ symmetric matrix. By taking a directional derivative of the objective in the direction $\matH$, we find a first order condition
  $
  \Tr\left[ D_\matH g(\matX) \matH\right] = 0,
  $
  with
  $
  D_\matH g(\matX) = -\matX^{-1}  + (\matX+ m \matI_m)^{-1} + \lambda \matK^{-1}.
  $
  Since this condition should be satisfied for all $\matH$, we simply solve the equation $D_\matH g(\matX) = 0$. Thanks to the invertibility of $\matK$, we have equivalently
  \[
  \matX^2 + m \matX - \frac{m}{\lambda}\matK = 0.
  \]
  A simple algebraic manipulation yields $ \matX_\star = \frac{1}{2}\left((m^2\matI_m + 4m  \matK /\lambda)^{1/2}-m\matI_m\right)\succ 0$. To check that $\matX_\star$ is minimum, it is sufficient to analyse the Hessian of the objective in any direction $\matH$. Let $t\geq 0$ small enough  such that $\matX + t \matH \succ 0$. The second order derivative of $g(\matX + t \matH)$ writes
  \[
   \frac{\rmd^2}{\rmd t^2} g(\matX + t \matH)|_{t=0} =   \Tr(\matH \matX^{-1} \matH \matX^{-1}) - \Tr(\matH (\matX+ m\matI_m)^{-1}\matH (\matX+ m\matI_m)^{-1}).
  \]
  By using the first order condition $\matX^{-1}_\star =  (\matX_\star + m\matI_m)^{-1} + \lambda \matK^{-1}$, we find
  \begin{align*}
   \frac{\rmd^2}{\rmd t^2} g(\matX_\star + t \matH)|_{t=0} &=   \Tr\Big(\matH(\lambda \matK^{-1})\matH \big(\lambda \matK^{-1} + 2 (\matX_\star + m\matI_m)^{-1} \big)\Big)\\
   & = \Tr\Big(\matM \matN \matM\Big).
  \end{align*}
  with $\matM = (\lambda \matK^{-1})^{1/2}\matH (\lambda \matK^{-1})^{1/2}$ and
  \[
    \matN = \matI + 2 (\lambda \matK^{-1})^{-1/2}(\matX_\star + m\matI_m)^{-1}(\lambda \matK^{-1})^{-1/2} \succ 0.
  \]
  Notice that $\matM \matN \matM\succeq 0$ since $\matN\succ 0$. Therefore, it holds $\frac{\rmd^2}{\rmd t^2} g(\matX_\star + t \matH)|_{t=0}\geq 0$. Since this is true for any direction $\matH$, the matrix $\matX_\star$ is a local minimum of $g(\matX)$.
  \end{proof}
While the exact solution of Proposition~\ref{prop:exact} is hard to intepret, if the regularization parameter goes to zero, the estimated correlation kernel tends to that of a projection DPP.
In what follows, the notation $f(\lambda)\sim g(\lambda)$ stands for $\lim_{\lambda\to+\infty } f(\lambda)/ g(\lambda) = 1$.
\begin{proposition}[Projection DPP in the low regularization limit]
  \label{prop:limit_as_projection}
  Under the assumptions of Proposition~\ref{prop:exact}, the correlation kernel $\Kbb(\lambda) = \opS \opL_\star(\lambda) \opS^*(\opS \opL_\star(\lambda) \opS^* + \I)^{-1},$ with  $\opL_\star(\lambda) = \sum_{i,j=1}^m \matC_{\star ij}(\lambda) \phi(x_i)\otimes \overline{\phi(x_j)}$
  and  $C_\star(\lambda)$ given in Proposition~\ref{prop:exact}, has a range that is independent of $\lambda$.
  Furthermore, $\Kbb(\lambda)$ converges to the projection operator onto that range when $\lambda\to 0$.
  In particular, the integral kernel of $\opL_\star(\lambda)$ has the following asymptotic expansion
  \[  \asf_\star(x,y) \sim \MF{(m/\lambda)^{1/2}}  \matk_x^\top \matK^{\MF{-3/2}} \matk_y \text{ as } \lambda\to 0,
  \]
  \MF{pointwisely}, with $\matk_x = [\kRKHS(x,x_1) \dots, \kRKHS(x,x_m)]^{\top}$.
\end{proposition}
\begin{proof}
  As shown in Lemma~\ref{lem:eigenvalues}, the non-zero eigenvalues of $A_\star(\lambda)$ are the eigenvalues of $\matK\matC_\star(\lambda)$ since $\matK$ is non-singular. \MF{The eigenvalues of $\matK\matC_\star(\lambda)$ are $\sigma_\ell(\lambda) = \frac{m}{2 \varsigma_\ell}(\sqrt{1+4\frac{\varsigma_\ell}{m\lambda}}-1)$ where $\varsigma_\ell$ is the $\ell$-th eigenvalue of $\matK$ for $1\leq \ell\leq m$.} Thus, the operator $A_\star(\lambda)$ is of finite rank, i.e., $A_\star(\lambda) = \sum_{\ell = 1}^m \sigma_\ell (\lambda) v_\ell \otimes \overline{v_\ell}$, where the normalized eigenvectors $v_\ell$ are independent of $\lambda$.  The eigenvalues satisfy $\sigma_\ell(\lambda) \to +\infty$ as $\lambda \to 0$, such that $\lim_{\lambda\to 0}\sigma_\ell(\lambda) \sqrt{\lambda}$ is finite. Note that
  $
    A_\star(\lambda) = \sum_{\ell = 1}^m \sigma_\ell (\lambda) (S v_\ell) \otimes \overline{(S v_\ell)}
  $
  is of finite rank and $S v_\ell$ does not depend on $\lambda$. Hence, the operator
  \[\Kbb(\lambda) = \sqrt{\lambda}S A_\star(\lambda) S^*\left(\sqrt{\lambda}S A_\star(\lambda) S^* + \sqrt{\lambda}\I\right)^{-1},\]
  converges to the projector on the range of $S A_\star(\lambda) S^*$ as $\lambda\to 0$.
\end{proof}

On the other hand, as the regularization parameter goes to infinity, the integral kernel of $\Lsf_\star(\lambda)$ behaves asymptotically as an out-of-sample Nystr\"om approximation.
\begin{proposition}[Large regularization limit]
  Under the assumptions of Proposition~\ref{prop:exact},  the integral kernel of $\Lsf_\star(\lambda)$ has the following asymptotic expansion
  \[  \asf_\star(x,y) \sim \lambda^{-1}  \matk_x^\top \matK^{-1} \matk_y \text{ as } \lambda\to +\infty,
  \]
  \MF{pointwisely}, with $\matk_x = [\kRKHS(x,x_1) \dots, \kRKHS(x,x_m)]^{\top}$.
\end{proposition}
The proof of this result is a simple consequence of the series expansion $\sqrt{1+4x} = 1 + 2x + O(x^2)$.
\section{Deferred proofs \label{supp:sec:Proofs}}
\subsection{Proof of Lemma~\ref{lemma:Bbar} \label{supp:lemma:Bbar:proof}}
\begin{proof}[Proof of Lemma~\ref{lemma:Bbar}]
Since $V^* V $ is the orthogonal projector onto the span of $ \phi(z_i)$, $1\leq i \leq m$, we have
\[
\matPhi_i^\top \bar{\matB} \matPhi_j = \langle V^* V \phi(x_i),  A V^* V\phi(x_j)\rangle = \langle \phi(x_i), A \phi(x_j)\rangle.
\]
By Sylvester's identity, and since $\|V^* V\|_{op}\leq 1$, it holds
\[
\log \det(\matI_m + \bar{\matB}) = \log \det(\I + V A V^*) = \log \det(\I + V^*V A)\leq \log \det(\I + A).
\]
Similarly, it holds that $\Tr(\bar{\matB}) = \Tr(V^*V A)\leq \Tr(A)$.
\end{proof}
\subsection{Approximation of the Fredholm determinant \label{supp:thm:formal_Fredholm:proof}}
\MF{Before giving the proof of Theorem~\ref{thm:formal_Fredholm}, we remind a well-known and useful formula for the Fredholm determinant of an operator. Let $M$ be a trace class endomorphism of a separable Hilbert space. Denote by $\lambda_\ell(M)$ its (possibly infinitely many) eigenvalues for $\ell\geq 1$, including multiplicities. Then, we have
\[
  \det(\I + M) = \prod_\ell (1+\lambda_\ell(M)),
\]
where the (infinite) product is uniformly convergent, see~\citet[Eq. (3.1)]{Bor10} and references therein. Now, let $\mathcal{F}_1$ and $\mathcal{F}_2$ be two separable Hilbert spaces. Let $T:\mathcal{F}_1 \to \mathcal{F}_2$ such that $T^* T$ is trace class. Note that $T^* T$ is self-adjoint and positive semi-definite. Then, $T^* T$ and $TT^*$ have the same non-zero eigenvalues (Jacobson's lemma) and Sylvester's identity
\[
  \det(\I + TT^*) = \det(\I + T^* T),
\]
holds. This can be seen as a consequence of the singular value decomposition of a compact operator; see~\citet[page 170, Theorem 7.6]{Weidmann1980}. We now give the proof of Theorem~\ref{thm:formal_Fredholm}.
}

\begin{proof}[Proof of Theorem~\ref{thm:formal_Fredholm}]
 To begin, we recall two useful facts. First, If $M$ and $N$ are two trace class endomorphisms of the same Hilbert space, then $\log \det (\I +  M )+ \log\det(\I + N) = \log\det\left((\I + M) (\I + N)\right)$; see e.g.\ \citet[Thm 3.8]{SIMON1977244}. Second, if $M$ is a  trace class endomorphisms of a Hilbert space, we have $(\I+M)^{-1} = \I - M (M+\I)^{-1}$, and thus, $\det((\I+M)^{-1})$ is a well-defined Fredholm determinant since $ M (M+\I)^{-1}$ is trace class. Now, define $\calR= \log \det (\I +  S A S^* )$ and $\calR_n = \log \det (\matI_n +  S_n A S_n^*)$.
Using Sylvester's identity and the aforementioned facts, we write
\begin{align}
  \calR_n - \calR  &= \log \det (\matI_n +  S_n A S_n^*)-\log \det (\I +  S A S^* ) \nonumber\\
    &= \log \det (\I +  A^{1/2}S^*_n S_n A^{1/2}) - \log \det (\I +  A^{1/2}S^* S A^{1/2}) \text{ (Sylvester's identity) }\nonumber\\
    &= \log \det \left((\I +  A^{1/2}S^*_n S_n A^{1/2}) (\I +  A^{1/2}S^* S A^{1/2})^{-1}\right).\label{eq:def_En}
\end{align}
By a direct substitution, we verify that $
    \calR_n -\calR = \log \det \left( \I + E_n\right),$
with \[ E_n =  (\I +  A^{1/2}S^*S A^{1/2})^{-1/2}A^{1/2}( S_n^*S_n - S^*S   )A^{1/2}(\I +  A^{1/2}S^*S A^{1/2})^{-1/2}.\]
Note that, in view of~\eqref{eq:def_En},  $\det \left( \I + E_n\right)$ is a positive real number.
Next, by using $S_n^*S_n - S^*S\preceq \| S_n^*S_n - S^*S \|_{op} \I$ and Sylvester's identity, we find
\begin{align*}
    \calR_n -\calR &\leq \log \det \left( \I +\| S_n^*S_n - S^*S \|_{op} (\I +  A^{1/2}S^*S A^{1/2})^{-1/2} A (\I +  A^{1/2}S^*S A^{1/2})^{-1/2}\right)\\
    &= \log \det \left( \I +\| S_n^*S_n - S^*S \|_{op} A^{1/2} (\I +  A^{1/2}S^*S A^{1/2})^{-1}A^{1/2}\right)\\
    &\leq \log\det(\I + \| S_n^*S_n - S^*S \|_{op} A).
\end{align*}
The latter bound is finite, since for $M$ a trace-class operator, we have $|\det(\I+M)|\leq \exp(\|M\|_\star)$, where $\|M\|_\star$ is the trace (or nuclear) norm.
By exchanging the roles of $S_n^*S_n$ and $S^*S $, we also find
\begin{align*}
   \calR -\calR_n \leq \log\det(\I + \| S_n^*S_n - S^*S \|_{op} A)
\end{align*}
and thus, by combining the two cases, we find
\begin{align*}
   |\calR -\calR_n | \leq \log\det(\I + \| S_n^*S_n - S^*S \|_{op} A).
\end{align*}
Now, in order to upper bound $\|S^*S - S_n^*S_n\|_{op}$ with high probability,
we use the following Bernstein inequality for a sum of random operators; see~\citet[Proposition~12]{rudi2015less} and~\citet{MINSKER2017111}.

\begin{proposition}[Bernstein’s inequality for a sum of i.i.d. random operators]
  \label{p:bernstein}
Let $\calH$ be a separable Hilbert space and let $X_1, \dots , X_n$ be a sequence of independent and identically distributed self-adjoint positive random operators on $\calH$.
Assume that  $\mathbb{E} X=0$ and that there exists a real number $\ell>0$ such that
$\lambda_{\max}(X)\leq \ell$ almost surely. Let $\Sigma$ be a trace class positive operator such that $\mathbb{E}(X^2)\preceq \Sigma$. Then, for any $\delta\in (0,1)$,
\begin{equation*}
\lambda_{\max}\left(\frac{1}{n}\sum_{i=1}^n X_i\right) \leq \frac{2\ell \beta}{3n} + \sqrt{\frac{2\|\Sigma\|_{op} \beta}{n}}, \quad \text{where } \beta=\log\left(\frac{2\Tr(\Sigma)}{\|\Sigma\|_{op}\delta}\right),
\end{equation*}
with probability $1-\delta$.
If there further exists an $\ell'$ such that $\|X\|_{op}\leq \ell'$ almost surely, then, for any $\delta\in (0,1/2)$,
\begin{equation*}
\left\|\frac{1}{n}\sum_{i=1}^n X_i\right\|_{op} \leq \frac{2\ell' \beta}{3n} + \sqrt{\frac{2\|\Sigma\|_{op} \beta}{n}}, \quad \text{where } \beta=\log\left(\frac{2\Tr(\Sigma)}{\|\Sigma\|_{op}\delta}\right),
  \end{equation*}
holds with probability $1-2\delta$.
\end{proposition}
To apply Proposition~\ref{p:bernstein} and conclude the proof, first recall the expression of the covariance operator \[
  \bbC = S^* S= \int_\calX \phi(x) \otimes \overline{\phi(x)} \rmd \mu(x),
\]
and of its sample version
\[
 S^*_n S_n= \frac{1}{n} \sum_{i=1}^n \phi(x_i) \otimes \overline{\phi(x_i)}.
\]
Define $X_i = \phi(x_i) \otimes \overline{\phi(x_i)} - \int_\calX \phi(x) \otimes \overline{\phi(x)} \rmd \mu(x)$. It is easy to check that $\mathbb{E}X_i = 0$ since $\mathbb{E}[\phi(x_i) \otimes \overline{\phi(x_i)}] = \int_\calX \phi(x) \otimes \overline{\phi(x)} \rmd \mu(x)$. Then, using the triangle inequality and that $\sup_{x\in \calX} \kRKHS(x,x) \leq \kappa^2$, we find the bound
\begin{align*}
\|X_i\|_{op} &= \left\|\int_\calX \left(\phi(x_i) \otimes \overline{\phi(x_i)} - \phi(x) \otimes \overline{\phi(x)}\right) \rmd \mu(x)\right\|_{op}\\
&\leq  \int_\calX \left(\left\| \phi(x_i) \otimes \overline{\phi(x_i)}\right\|_{op} + \left\|\phi(x) \otimes \overline{\phi(x)}\right\|_{op} \right) \rmd \mu(x)\\
& \leq 2 \kappa^2 \triangleq \ell'.\\
\end{align*}
Next, we compute a bound on the variance by bounding the second moment, namely
\begin{align*}
    \E(X_i^2) &= \E\left[\left(\phi(x_i) \otimes \overline{\phi(x_i)}\right)^2\right] - \left(\mathbb{E}\left[\phi(x_i) \otimes \overline{\phi(x_i)}\right]\right)^2\\
    &\preceq \E\left[\left(\phi(x_i) \otimes \overline{\phi(x_i)}\right)^2\right]\\
    &\preceq \kappa^2 \mathbb{E}\left[\phi(x_i) \otimes \overline{\phi(x_i)}\right]\triangleq \Sigma,
\end{align*}
where we used $\left(\phi(x_i) \otimes\overline{\phi(x_i)}\right)^2 = \kRKHS(x_i,x_i) \left(\phi(x_i) \otimes \overline{\phi(x_i)}\right)$.
Also, we have
\[
\Tr(\Sigma) = \Tr\left(\kappa^2 \mathbb{E}\left[\phi(x) \otimes \overline{\phi(x)}\right]\right) = \kappa^2 \int_\calX \kRKHS(x,x) \rmd \mu(x)\leq \kappa^4.
\]
Moreover,
\[
\|\Sigma\|_{op} = \kappa^2 \left\| \mathbb{E}\left[\phi(x) \otimes \overline{\phi(x)}\right]\right\|_{op} =  \kappa^2 \lambda_{\max}(S^* S)=  \kappa^2 \lambda_{\max}(\Lk),
\]
where we used that $S^* S = \Lk$ is the integral operator on $L^2(\calX)$ with integral kernel $k$.
We are finally ready to apply Proposition~\ref{p:bernstein}.
Since
\[
\beta \leq \log\left( \frac{2\kappa^2}{\lambda_{\max}(\Lk) \delta}\right),
\]
it holds, with probability at least $1-2\delta$,
\[
\|S^*S - S^*_n S_n\|_{op}\leq \frac{4\kappa^2 \log\left( \frac{2\kappa^2}{\lambda_{\max}(\Lk) \delta}\right)}{3n} + \sqrt{\frac{2\kappa^2 \lambda_{\max}(\Lk) \log\left( \frac{2\kappa^2}{\lambda_{\max}(\Lk) \delta}\right)}{n}}.
\]
This concludes the proof of Theorem~\ref{thm:formal_Fredholm}.
\end{proof}
\subsection{Statistical guarantee for the approximation of the log-likelihood by its sampled version \label{supp:thm:bound_objectives_whp:proof}}
\begin{proof}[Proof of Theorem~\ref{thm:bound_objectives_whp}]
The proof follows similar lines as in~\citet[Proof of Thm~5]{Rudi2020global}, with several adaptations.
Let $\matB_\star \in \Sb(\R^m)$ be the solution of~\eqref{eq:SDP_B}. Notice that $A_\star$ and $V^*\matB_\star V$ are distinct operators. Let $\bar{\matB} = V A_\star V^* \in \Sb(\R^m)$ as in Lemma~\ref{lemma:Bbar}.
Since $\matB_\star $ has an optimal objective value, we have
\[
f_n(V^*\matB_\star  V) + \Tr(\lambda \matB_\star ) \leq f_n(V^*\bar{\matB}V) +\Tr(\lambda \bar{\matB}).
\]
Now we use the properties of $\bar{\matB}$ given in Lemma~\ref{lemma:Bbar}, namely that $f_n(A_\star) = f_n(V^* \bar{\matB} V)$ and $\Tr(\lambda \bar{\matB}) \leq  \Tr(\lambda A_\star)$. Then it holds
\[
 f_n(V^*\bar{\matB} V) + \Tr(\lambda \bar{\matB}) \leq f_n(A_\star) + \Tr(\lambda A_\star).
\]
By combining the last two inequalities, we have
$
f_n(V^* \matB_\star  V) + \Tr(\lambda \matB_\star ) \leq f_n(A_\star) + \Tr(\lambda A_\star)
$
and therefore
\begin{equation}
    f_n(V^* \matB_\star  V) - f_n(A_\star) \leq \Tr( \lambda A_\star)-\Tr(\lambda \matB_\star ).\label{eq:key}
\end{equation}
We will use \eqref{eq:key} to derive upper and lower bounds on the gap $f(A_\star) - f_n(V^*\matB_\star  V)$.
\paragraph{Lower bound.} Using Theorem~\ref{thm:formal_Fredholm}, we have with high probability
\[
|f(A_\star) - f_n(A_\star)| \leq \log \det (\I + c_n A),
\]
and, in particular, this gives
\[
f(A_\star) - f_n(A_\star)\leq \Tr(c_n A).
\]
By combining this last inequality with~\eqref{eq:key} and by using $\Tr(\lambda \matB_\star )\geq 0$, we have the lower bound
\begin{align}
\Delta = f(A_\star) - f_n(V^*\matB_\star  V) &=f(A_\star) - f_n(A_\star) + f_n(A_\star)- f_n(V^* \matB_\star  V)\nonumber\\
& \geq -\Tr( c_n A_\star) + \Tr(\lambda \matB_\star ) -\Tr( \lambda A_\star)\label{eq:key2}\\
& \geq -( c_n + \lambda)\Tr(A_\star).\nonumber
\end{align}
\paragraph{Upper bound.}
We have the bound with high probability (for the same event as above)
\begin{align}
\Delta = f(A_\star) - f_n(\matB_\star ) &= \underbrace{f(A_\star)-f(V^* \matB_\star  V)}_{\leq 0} + f(V^* \matB_\star  V) - f_n(V^*\matB_\star  V)\nonumber\\
&\leq \log\det(\I + c_n V^* \matB_\star  V) \text{ \quad (Theorem~\ref{thm:formal_Fredholm})}\nonumber\\
&\leq \Tr(c_n V V^* \matB_\star  ) \text{ \quad (cyclicity)}\nonumber\\
&\leq \Tr(c_n \matB_\star  ). \text{ \quad (since $V V^*$ is a projector)} \label{eq:Next2LastStepProof}
\end{align}
By combining this with~\eqref{eq:key2}, we find
\[
 -\Tr( c_n A_\star) + \Tr(\lambda \matB_\star  ) -\Tr( \lambda A_\star) \leq \Tr(c_n \matB_\star ).
\]
Since, by assumption, we have $c_n \leq \lambda - c_n$, the latter inequality becomes
\begin{equation}
c_n \Tr(\matB_\star  ) \leq (c_n + \lambda) \Tr( A_\star ).\label{eq:B_bound}
\end{equation}
By using~\eqref{eq:B_bound} in the bound in~\eqref{eq:Next2LastStepProof}, we obtain
\[
f(A_\star) - f_n(V^* \matB_\star  V) \leq (c_n + \lambda) \Tr( A_\star ),
\]
Thus, the upper and lower bound yield together
\[
|f(A_\star) - f_n(V^* \matB_\star  V)| \leq (c_n + \lambda) \Tr( A_\star )\leq \left(\frac{\lambda}{2} + \lambda\right) \Tr( A_\star ),
\]
where we used once more $2 c_n \leq \lambda$.
This is the desired result.
\end{proof}

\begin{proof}[Proof of Corollary~\ref{corol:likelihood_approximation}]
  We use the triangle inequality
  \[
  | f(A_\star) - f(V\matB_\star  V^*)| \leq | f(A_\star)-f_n(V\matB_\star  V^*)| + |f_n(V\matB_\star  V^*) - f(V\matB_\star  V^*)|.
  \]
  The first term is upper bounded whp by Theorem~\ref{thm:bound_objectives_whp}. The second term is bounded by Theorem~\ref{thm:formal_Fredholm} as follows
  \[
  |f_n(V\matB_\star  V^*) - f(V\matB_\star  V^*)|\leq \Tr(c_n \matB_\star )
  \]
  with $\Tr(c_n \matB_\star )\leq \frac{3}{2}\lambda \Tr(A_\star)$ as in~\eqref{eq:B_bound} in the proof of Theorem~\ref{thm:bound_objectives_whp}.
  \end{proof}
\subsection{Numerical approach: convergence study \label{supp:thm:Picard:proof}}
\begin{proof}[Proof of Theorem \ref{thm:Picard}]
  This proof follows the same technique as in~\cite{pmlr-v37-mariet15}, with some extensions. Let $\matSigma = \matX^{-1}$. We decompose the objective
  \begin{align*}
    g(\matSigma^{-1}) &= \log\det\left( \matI_m + \frac{1}{|\calI|}\matSigma^{-1}  \right)_{\calI\calI} -\frac{1}{s}\sum_{\ell= 1}^{s}\log\det(\matU_\ell^\top \matSigma^{-1} \matU_\ell) + \lambda\Tr(  \matSigma^{-1} \matK^{-1})
  \end{align*}
  as the following sum: $g(\matSigma^{-1}) = h_1(\matSigma) + h_2(\matSigma)$, where
  $
    h_1(\matSigma) = -\log\det(\matSigma) + \lambda\Tr(  \matSigma^{-1} \matK^{-1})
  $
  is a strictly convex function on $\matSigma\succ 0$ and
  \[
    h_2(\matSigma) = \log\det(\matSigma) + \log\det\left(\matI_m+\frac{1}{|\calI|}\matSigma^{-1}\right)_{\calI\calI}-\frac{1}{s}\sum_{\ell= 1}^{s}\log\det(\matU_\ell^\top \matSigma^{-1} \matU_\ell)
  \]
  is concave on $\matSigma\succ 0$. We refer to Lemma~\ref{lem:concavity_h1} and~\ref{lem:concavity_h2} for the former and latter statements, respectively. Then, we use the concavity of $h_2$ to write the following upper bound
  \[
    h_2(\matSigma)\leq h_2(\matY) + \Tr\Big(\nabla h_2(\matY) (\matSigma-\matY)\Big),
  \]
  where the matrix-valued gradient is
  \begin{align*}
    \nabla h_2(\matY) &= \matY^{-1} -  \matY^{-1}\matU_\calI \left(|\calI| \matI_{| \calI |}+ \matU_\calI^\top \matY^{-1} \matU_\calI\right)^{-1}\matU_\calI^\top \matY^{-1} \\
    &+ \frac{1}{s}\sum_{\ell = 1}^s \matY^{-1} \matU_\ell \left(\matU_\ell^\top \matY^{-1} \matU_\ell\right)^{-1}\matU_\ell^\top \matY^{-1} .
  \end{align*}
  Define the auxillary function $\xi(\matSigma, \matY) \triangleq h_1(\matSigma) + h_2(\matY) + \Tr\left(\nabla h_2(\matY) (\matSigma-\matY)\right)$ which satisfies $g(\matSigma^{-1})\leq \xi(\matSigma, \matY)$ and $g(\matSigma^{-1}) =  \xi(\matSigma, \matSigma)$. We define the iteration $\matX_k = \matSigma_k^{-1}$ where
  \[
  \matSigma_{k+1} =  \arg\min_{\matSigma\succ 0} \xi(\matSigma,\matSigma_k),
  \]
  so that it holds $g(\matSigma^{-1}_{k+1})\leq \xi(\matSigma_{k+1}, \matSigma_k)\leq \xi(\matSigma_{k}, \matSigma_k) = g(\matSigma^{-1}_{k})$. Thus, this iteration has monotone decreasing objectives. It remains to show that this iteration corresponds to~\eqref{eq:regPicard}. The solution of the above minimization problem can be obtained by solving the first order optimality condition since $\xi(\cdot,\matY)$ is strictly convex. This gives
  \begin{align*}
    &-\matSigma^{-1} - \lambda \matSigma^{-1} \matK^{-1} \matSigma^{-1}
    + \matSigma^{-1}_k -   \matSigma^{-1}_k \matU_\calI \left(|\calI|\matI_{|\calI|}+ \matU_\calI^\top\matSigma^{-1}_k \matU_\calI\right)^{-1}\matU_\calI^\top \matSigma^{-1}_k \\
    &+ \frac{1}{s}\sum_{\ell = 1}^s \matSigma^{-1}_k \matU_\ell \left(\matU_\ell^\top \matSigma^{-1}_k \matU_\ell\right)^{-1}\matU_\ell^\top \matSigma^{-1}_k  = 0.
  \end{align*}
  Now, we replace $\matX = \matSigma^{-1}$ in the above condition. After a simple algebraic manipulation, we obtain the following condition
  \[
    \matX + \lambda \matX \matK^{-1} \matX =   p(\matX_k),
  \]
  where, as defined in~\eqref{eq:regPicard}, $p(\matX) = \matX+ \matX\matDelta \matX$  and
  \[
    \matDelta = \frac{1}{s}\sum_{\ell= 1}^{s} \matU_\ell \matX_{\calC_\ell \calC_\ell}^{-1}\matU_\ell^\top  - \matU_\calI \left(|\calI|\matI_{|\calI|}+ \matU_\calI^\top \matX \matU_\calI\right)^{-1}\matU_\calI^\top.
  \] Finally, we introduce the Cholesky decomposition $\matK = \matR^\top \matR$, so that we have an equivalent identity
  \[
    \matR^{-1\top} \matX \matR^{-1} + \lambda \left(  \matR^{-1\top} \matX \matR^{-1}\right)^2 - \matR^{-1\top} p(\matX_k) \matR^{-1} = 0.
  \]
  Let $\matX' = \matR^{-1\top} \matX \matR^{-1}$ and $p'(\matX_k) = \matR^{-1\top} p(\matX_k) \matR^{-1}$. The positive definite solution of this second order matrix equation write
  \[
    \matX' = \frac{-\matI_m + \left(\matI_m + 4 \lambda  p'(\matX_k) \right)^{1/2}}{2\lambda},
  \]
  which directly yields~\eqref{eq:regPicard}.
  \end{proof}
\subsection{Approximation of the correlation kernel \label{supp:thm:approx_correlation_kernel_simplified:proof}}
We start by proving the following useful result.
\begin{theorem}[Correlation kernel approximation, formal version]\label{thm:approx_correlation_kernel}
  Let $\delta \in (0,1]$ be the failure probability and let $\gamma>0$ be a scale factor. Let $\hat{\Kbb} (\gamma)$ be the approximation defined in~\eqref{eq:K_approx} with i.i.d. sampling of $p$ points. Let $p$ be large enough so that $t(p)>1$ with  $t(p) = \frac{4c^2 \beta}{3\gamma p}+\sqrt{\frac{2c^2 \beta}{\gamma p}}$ where  $c^2 = \kappa^2 \|A\|_{op}$ and
  $
    \beta = \log\left(\frac{4 d_{\rm eff}(\gamma)}{\delta\|\Kbb(\gamma)\|_{op}}\right) 
  $  Then, with probability $1-\delta$ it holds that
  \[
    \frac{1}{1+t(p)}\Kbb(\gamma)\preceq \hat{\Kbb} (\gamma) \preceq \frac{1}{1-t(p)} \Kbb(\gamma).
  \]
  Furthermore, if we assume $\gamma\leq \lambda_{\max}(\Lsf)$, we can take
  $
    \beta = \log\left(\frac{8 d_{\rm eff}(\gamma)}{\delta}\right)\leq \frac{8 d_{\rm eff}(\gamma)}{\delta}.
  $
  \end{theorem}
  \begin{proof}
    For simplicity, define $\Psi: \calH \to \R^m$ as $\Psi = \sqrt{m}\matLambda S_m$, which is such that $A_\star = \Psi^* \Psi$. Then, we write
    \[
      \hat{\Kbb}  = S \Psi^* (\Psi S^*_p S_p \Psi^* + \gamma \matI_m)^{-1}\Psi S^*,
    \]
    where we recall that $S^*_p S_p = \frac{1}{p}\sum_{i=1}^p \phi(x''_i)\otimes \overline{\phi(x''_i)}$.
    Next, we used the following result.
    \begin{proposition}[Proposition~5 in \cite{rudi2018fast} with minor adaptations]\label{prop:Bernstein_effective}
      Let $\gamma >0$ and $v_1, \dots, v_p$ with $p\geq 1$ be identically distributed random vectors on a separable Hilbert space $H$, such that there exists $c^2>0$ for which $\|v\|_H\leq c^2$ almost surely. Denote by $Q$ the Hermitian operator $Q=\frac{1}{p}\sum_{i=1}^p \E[v_i \otimes \overline{v_i}]$.
      Let $Q_p = \frac{1}{p}\sum_{i=1}^p v_i \otimes \overline{v_i}$. Then for any $\delta\in (0,1]$, the following holds
      \[
        \|(Q + \gamma \mathbb{I})^{-1/2}(Q-Q_p)(Q+\gamma \mathbb{I})^{-1/2}\|_{op} \leq \frac{4c^2 \beta}{3\gamma p}+\sqrt{\frac{2c^2 \beta}{\gamma p}},
      \]
      with probability $1-\delta$ and
      $\beta = \log\frac{4 \Tr\left(Q(Q+\gamma\I)^{-1}\right)}{\delta\|Q(Q+\gamma\I)^{-1}\|_{op}}\leq 8 \frac{c^2/\|Q\|_{op}+\Tr\left(Q(Q+\gamma\I)^{-1}\right)}{ \delta}$.
    \end{proposition}
    Then, we merely define the following vector
    $\matv_i = \Psi \phi(x''_i)$ for $ 1\leq i \leq p$ so that
    \[
      \matQ_p = \Psi S_p^* S_p \Psi^* = \frac{1}{p}\sum_{i=1}^p\Psi\left(\phi(x''_i)\otimes \overline{\phi(x''_i)}\right)\Psi^*.
    \]
    Furthermore, we define $\matQ=  \Psi S^* S\Psi^*$. Also, we have $S^* S = \int_\calX \phi(x)\otimes \overline{\phi(x)}\rmd\mu(x)$, so that $\E[S^*_p S_p] = S^* S$. Hence, it holds $ \E[\matQ_p] = \matQ$. First, by using $\Psi^* \Psi \preceq \|\Psi^* \Psi \|_{op} \I$ , we have
    \[
      \|\matv\|^2_2 = \langle \phi(x), \Psi^* \Psi \phi(x)\rangle\leq \kRKHS(x,x)\|\Psi^* \Psi\|_{op} \leq \kappa^2 \|\Psi^* \Psi\|_{op} = \kappa^2 \|A_\star\|_{op},
    \]
    almost surely.
    Next, we calculate the following quantity
    \begin{align*}
      \Tr\left[\matQ(\matQ+\gamma\matI_m)^{-1}\right]
      &= \Tr\left[ \Psi S^* S\Psi^* (\Psi S^* S\Psi^*+\gamma \matI_m)^{-1} \right]\\
      &= \Tr\left[ S \Psi^* (\Psi S^* S \Psi^*+\gamma \matI_m)^{-1}\Psi S^*\right] \\
      &= \Tr\left[ S \Psi^* \Psi S^*( S \Psi^*\Psi S^*+\gamma \mathbb{I})^{-1}  \right] \\
      &= \Tr\left[\Lsf(\Lsf + \gamma \I)^{-1}\right],
    \end{align*}
    where we used the push-through identity at the next to last equality.

    For obtaining the bound on $\beta$, we first write
    \[
      \|\matQ(\matQ+\gamma\matI_m)^{-1}\|_{op} = \|\Lsf (\Lsf + \gamma \I)^{-1}\|_{op} = (1+ \gamma/\lambda_{\max}(\Lsf))^{-1}.
    \]
    To lower bound the latter quantity we require $\gamma\leq \lambda_{\max}(\Lsf)$ and hence $\|\matQ(\matQ+\gamma\matI_m)^{-1}\|_{op}\geq 1/2$.
    For the remainder of the proof, we show the main matrix inequality.
    For convenience, define the upperbound in Proposition~\ref{prop:Bernstein_effective} as
    \begin{equation}
      t(p) = \frac{4c^2 \beta}{3\gamma p}+\sqrt{\frac{2c^2 \beta}{\gamma p}}.\label{eq:t(p)}
    \end{equation}
    Thanks to Proposition~\ref{prop:Bernstein_effective}, we know that with probability $1-\delta$, we have
    \[
      -t \left(\Psi S S^* \Psi^* + \gamma \matI_m\right)\preceq \Psi S S^* \Psi^* - \Psi S_p^* S_p \Psi^* \preceq t \left(\Psi S S^* \Psi^* + \gamma \matI_m\right),
    \]
    or equivalently
    \[
      \Psi S S^* \Psi^* - t (\Psi S S^* \Psi^* + \gamma\matI_m)\preceq  \Psi S_p^* S_p \Psi^* \preceq \Psi S S^* \Psi^* + t (\Psi S S^* \Psi^* +\gamma\matI_m).
    \]
    By simply adding $\gamma\matI_m$ to these inequalities, we obtain
    \[
      (1-t)(\Psi S S^* \Psi^* + \gamma\matI_m)\preceq  \Psi S_p^* S_p \Psi^* +  \gamma\matI_m \preceq(1 + t) (\Psi S S^* \Psi^* +\gamma\matI_m).
    \]
    Hence, if $t<1$, by a simple manipulation, we find
    \[
      (1+t)^{-1}(\Psi S S^* \Psi^* + \gamma\matI_m)^{-1} \preceq (\Psi S^*_p S_p \Psi^* +\gamma \matI_m)^{-1}\preceq (1-t)^{-1} (\Psi S S^* \Psi^* + \gamma\matI_m)^{-1}.
    \]
    By acting with $S \Psi^*$ on the left and $\Psi S^* $ on the right, and then, using the push-through identity
    \[
      S \Psi^*(\Psi S S^* \Psi^* + \gamma\matI_m)^{-1}\Psi S^* = (S \Psi^*\Psi S  + \gamma\I)^{-1}S^* \Psi^*\Psi S^*,
    \]
    the desired result follows.
    \end{proof}

    We can now prove Theorem~\ref{thm:approx_correlation_kernel_simplified} by simplifying some of the bounds given in Theorem~\ref{thm:approx_correlation_kernel}.
    \begin{proof}[Proof of Theorem~\ref{thm:approx_correlation_kernel_simplified}]
      Consider the upper bound given in~\eqref{eq:t(p)}. We will simplify it to capture the dominant behavior as $p\to +\infty$. Assume $\sqrt{\frac{2c^2 \beta}{\gamma p}}<1$, or equivalently
      $
        p> \frac{2c^2 \beta}{\gamma}.
      $
      In this case, we give a simple upper bound on $t(p)$ as follows
      \[
        t(p)< \left(\frac{2}{3} + 1\right) \sqrt{\frac{2c^2 \beta}{\gamma p}}< \sqrt{\frac{8c^2 \beta}{\gamma p}},
      \]
      so that we avoid manipulating cumbersome expressions. Thus, if we want the latter bound to be smaller than $\epsilon \in (0,1)$, we require
      \[
        p \geq \frac{8c^2 \beta}{\gamma \epsilon^2},
      \]
      which is indeed larger than $\frac{2c^2 \beta}{\gamma}$ since $1/\epsilon> 1$. Thus, by using the same arguments as in the proof of Theorem~\ref{thm:approx_correlation_kernel}, we have the multiplicative error bound
      \[
        \frac{1}{1+\epsilon}\Kbb(\gamma)\preceq \hat{\Kbb} (\gamma) \preceq \frac{1}{1-\epsilon} \Kbb(\gamma),
      \]
      with probability at least $1-\delta$ if
      \[
        p \geq \frac{8c^2 \beta}{\gamma \epsilon^2} = \frac{8\kappa^2 \| A\|_{op}}{\gamma \epsilon^2}\log\left( \frac{4 d_{\rm eff}(\gamma)}{\delta \| \Kbb \|_{op}}\right)
      \]
      where the last equality is obtained by substituting
      $c^2 = \kappa^2 \|A\|_{op}$ and
    $
    \beta = \log\left(\frac{4 d_{\rm eff}(\gamma)}{\delta\|\Kbb(\gamma)\|_{op}}\right)
    $
    given in Theorem~\ref{thm:approx_correlation_kernel}.
    \end{proof}
    \section{Supplementary empirical results \label{supp:OtherSimulations}}
    \subsection{Finer analysis of the Gaussian L-ensemble estimation problem of Section~\ref{sec:simulations}}
    In this section, we report results corresponding to the simulation setting of Section~\ref{sec:simulations} with $\rho = 100$.
    \paragraph{Intensity estimation from several DPP samples.}
    In Figure~\ref{supp:fig:s=10Intensity}, we replicate the setting of Figure~\ref{fig:Intensity1Sample} with $s=3$ and $s = 10$ DPP samples and a smaller regularization parameter. The estimated intensity is then closer to the ground truth ($\rho = 100$) for a large value of $s$, although there are small areas of high intensity at the boundary of the domain $[0,1]^2$. A small improvement is also observed by increasing $s$ from $3$ (left) to $10$ (right), namely the variance of the estimated intensity tends to decrease when $s$ increases.
    \begin{figure}[h!]
      \centering
      \includegraphics[scale = 0.3]{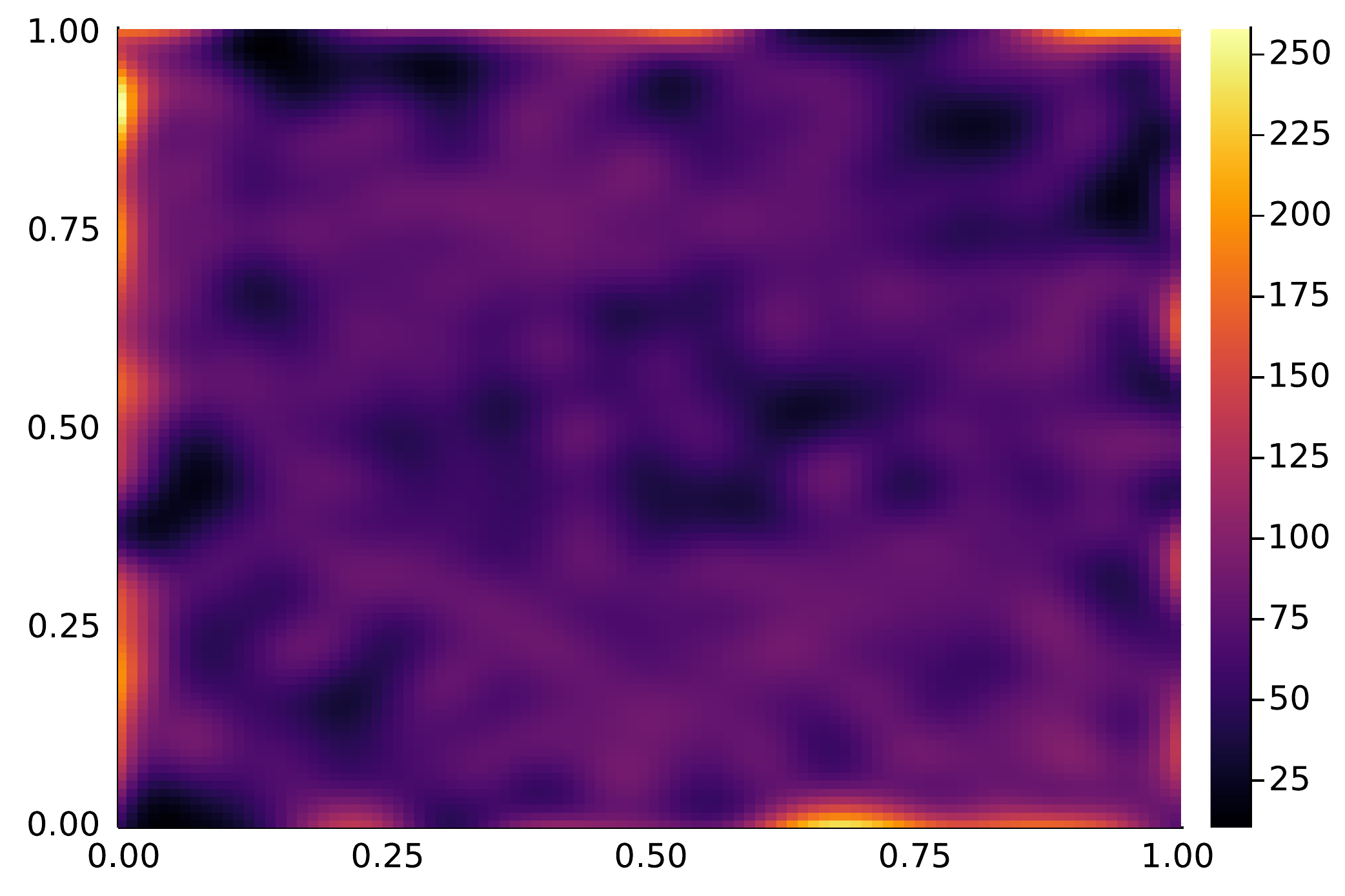}
      \hfill
      \includegraphics[scale = 0.3]{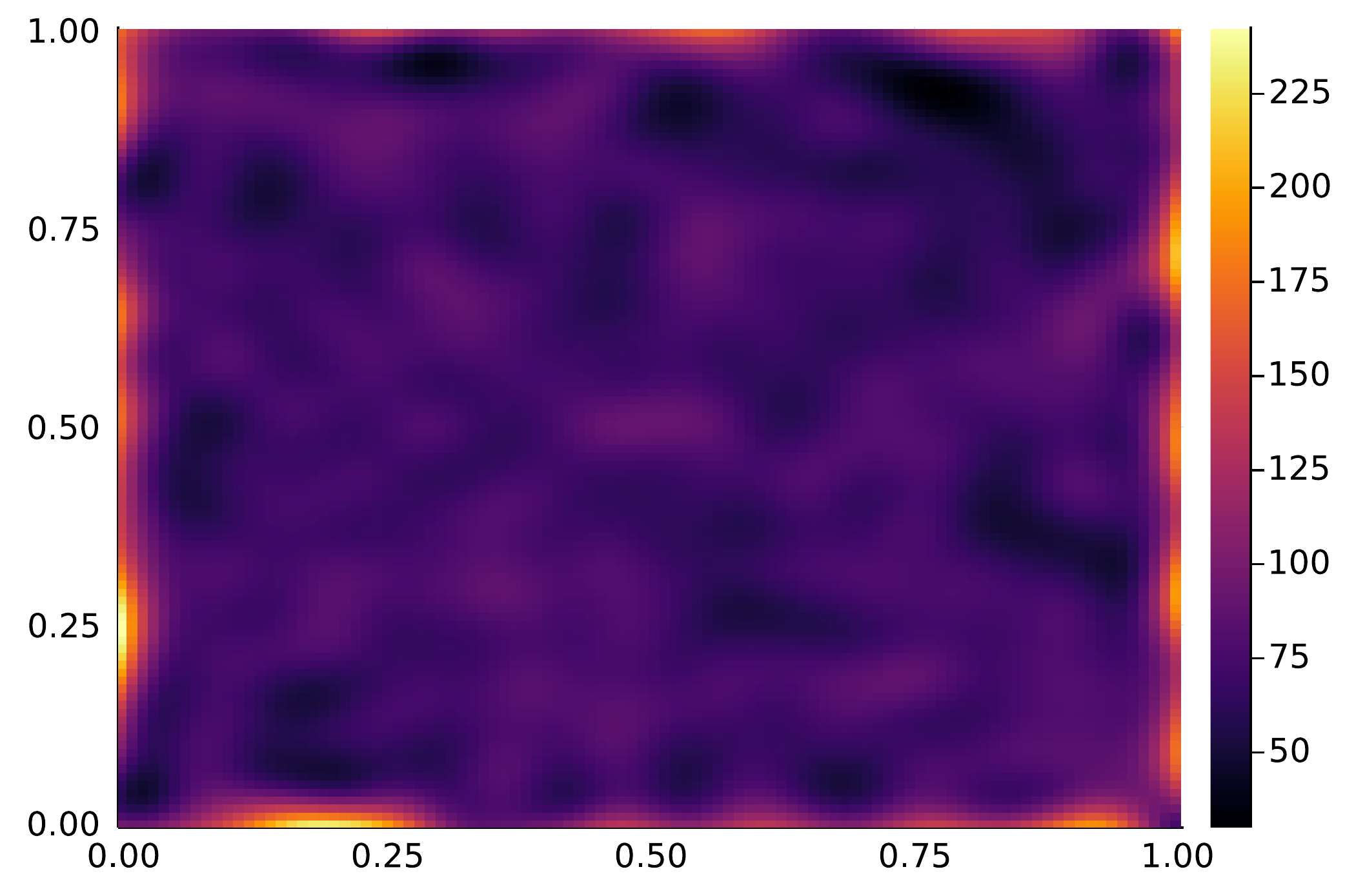}
      \includegraphics[scale = 0.3]{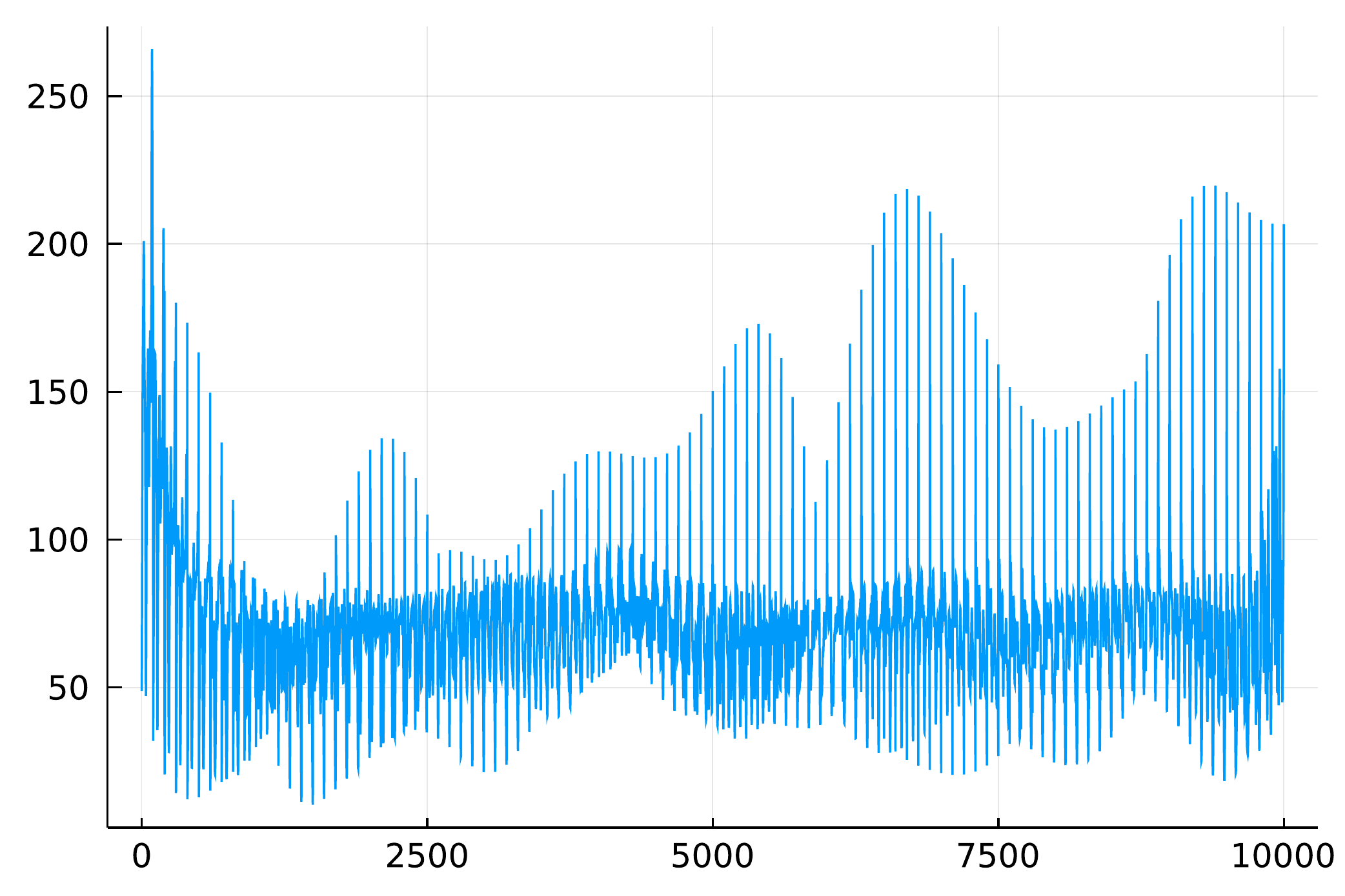}
      \hfill
      \includegraphics[scale = 0.3]{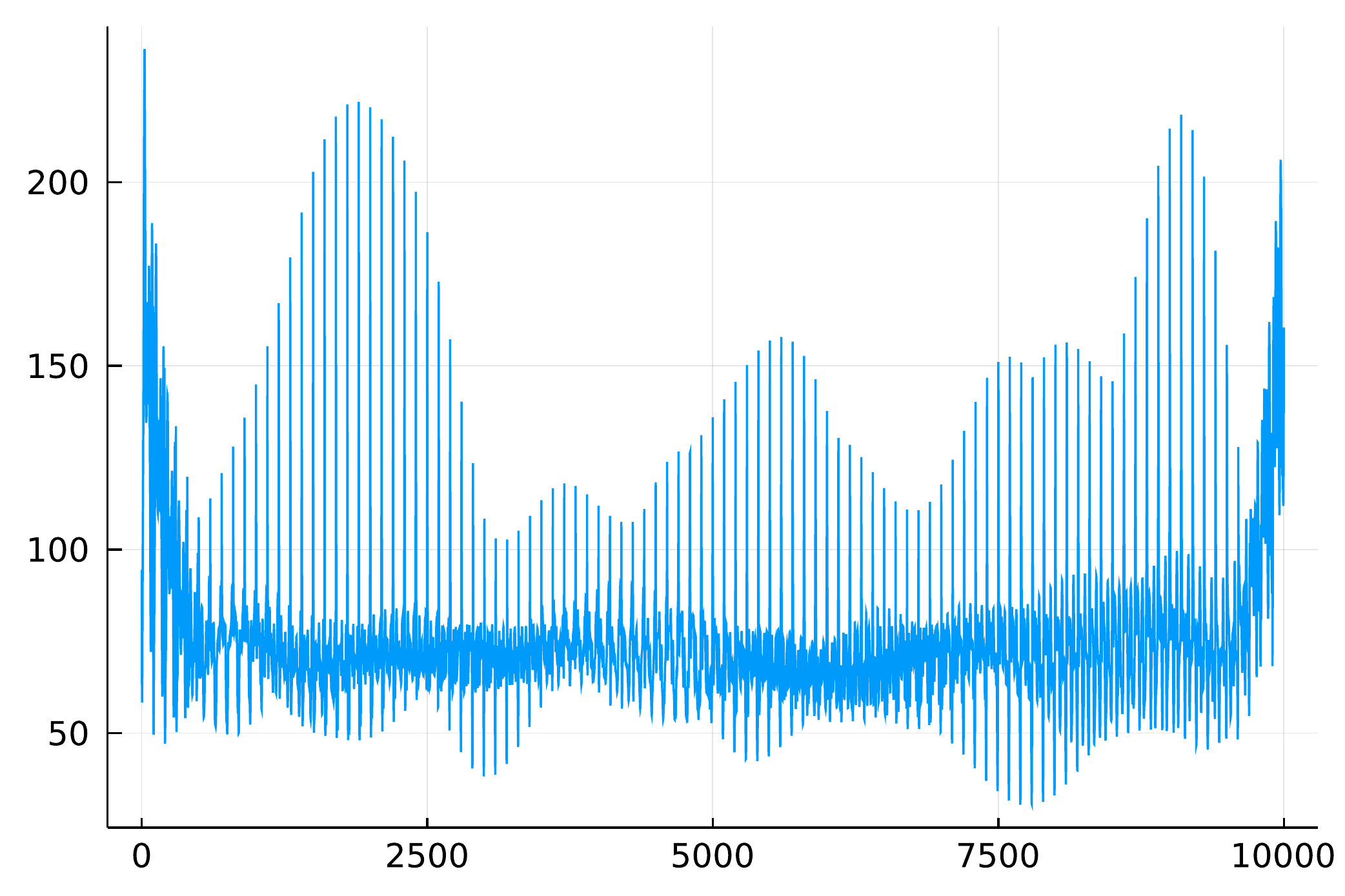}
      \caption{Effet of the number of samples on intensity estimation (with $\sigma = 0.1$ and $\lambda= 10^{-4}$, as in Figure~\ref{fig:Intensity1Sample}). Left column: estimation from $s=3$ DPP samples. Right column: estimation from $s=10$ DPP samples. The first row is a heatmap of the intensity $\hat{\ksf}(x,x)$ on a $100\times 100$ grid within $[0,1]^2$. \MF{The second row is the same data matrix in a flattened format, that is, each column of the $100\times 100$ data matrix is concatenated to form a $10000\times 1$ matrix whose entries are plotted.} Notice that the sharp peaks are due to boundary effects. \MF{These peaks are regularly spaced due to the column-wise unfolding.}\label{supp:fig:s=10Intensity}}
    \end{figure}
    \begin{figure}[h!]
      \centering
      \includegraphics[scale = 0.3]{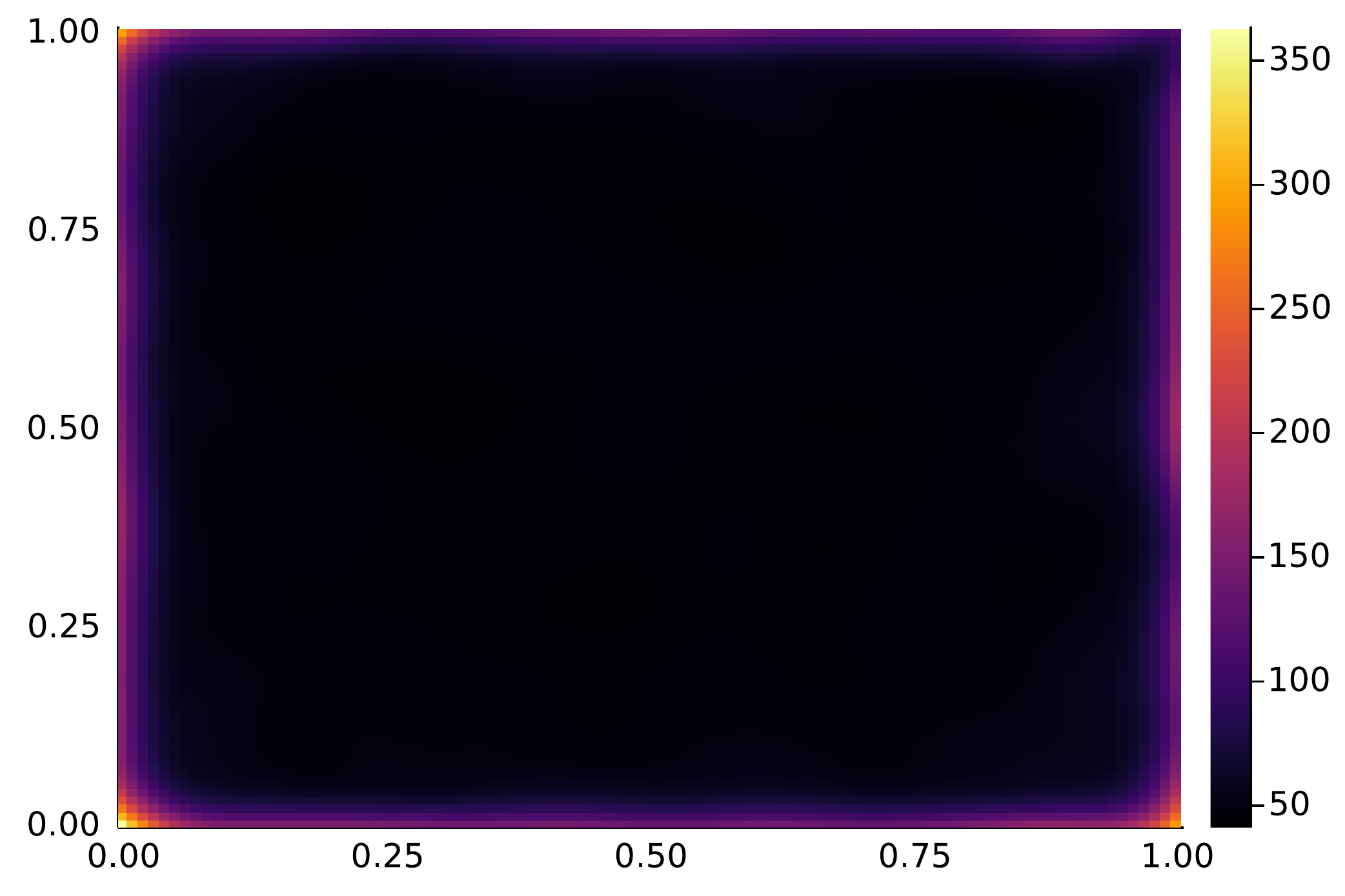}
      \hfill
      \includegraphics[scale = 0.3]{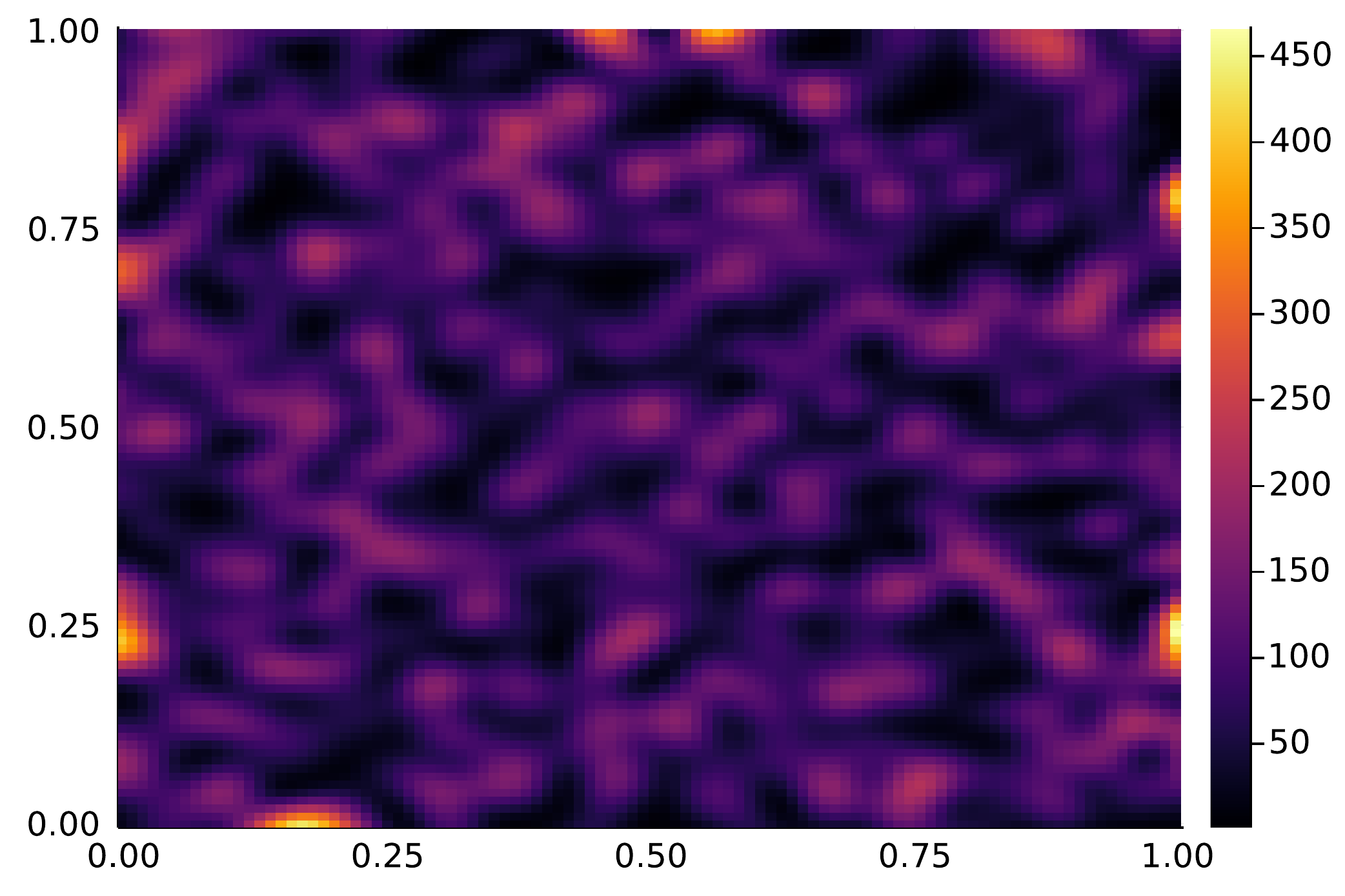}
      \hfill
      \includegraphics[scale = 0.3]{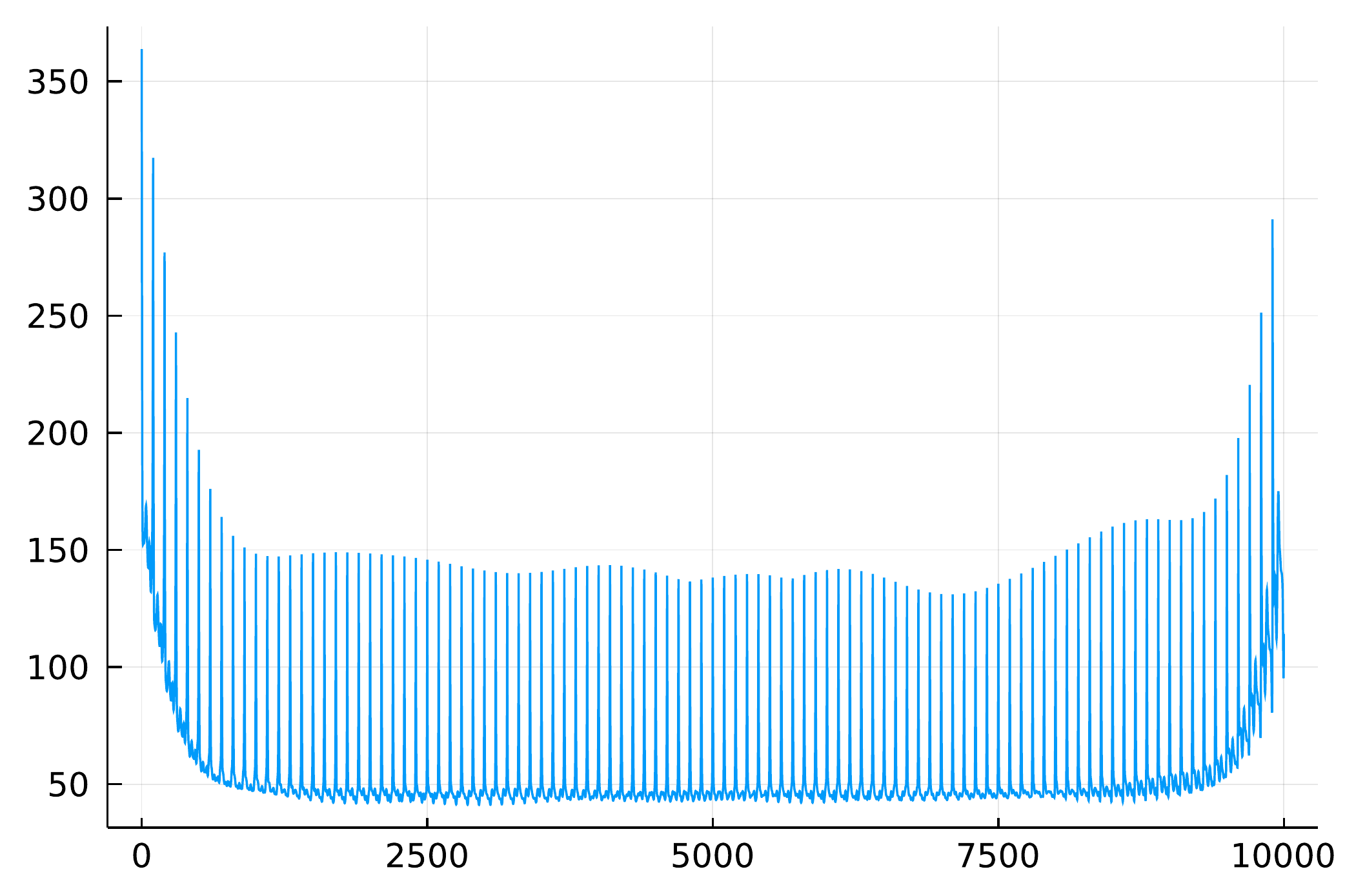}\hfill
      \includegraphics[scale = 0.3]{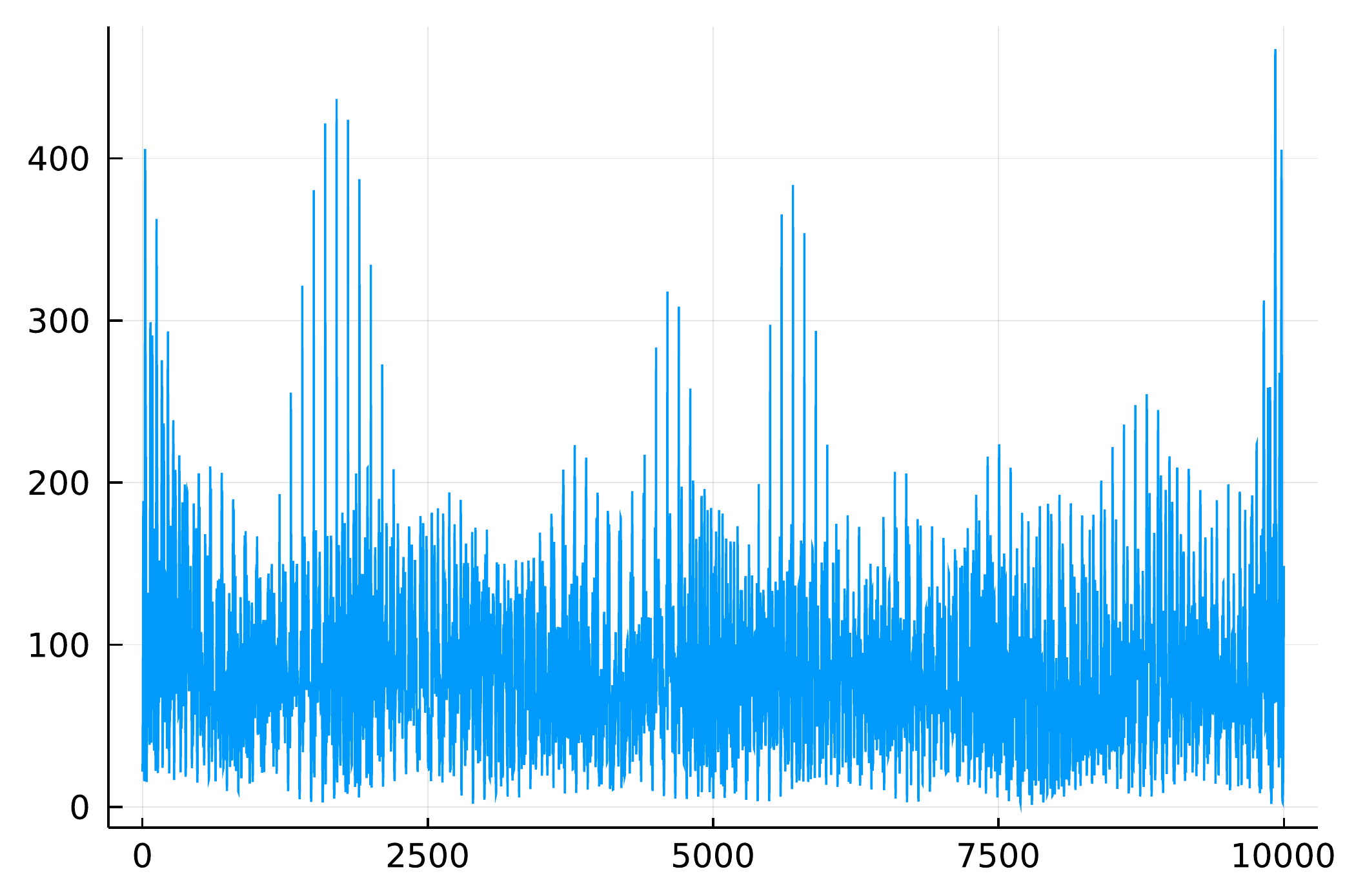}
      \caption{Effet of the bandwidth $\sigma$ on intensity estimation. Left column:  large value ($\sigma = 0.15$, $\lambda= 10^{-4}$) with $s = 10$ DPP samples.  Right column:
      small value ($\sigma = 0.05$, $\lambda= 10^{-4}$) with $s = 3$ DPP samples. The first row is a heatmap of the intensity on $[0,1]^2$. \MF{The second row is the same data matrix in a flattened format, that is, each column of the $100\times 100$ data matrix is concatenated to form a $10000\times 1$ matrix whose entries are plotted.}  Notice that the sharp peaks at the bottom row are due to boundary effects. \MF{These peaks are regularly spaced due to the column-wise unfolding.}\label{supp:fig:s=10IntensitysigmaLargeSmall}}
    \end{figure}
    In Figure~\ref{supp:fig:s=10IntensitysigmaLargeSmall}, we illustrate the intensity estimation in the case of a large and small $\sigma$, respectively on the left and right columns. As expected, a large value of $\sigma$ has a regularization effect but also leads to an underestimation of the intensity. On the contrary, a small value of $\sigma$ seems to cause inhomogeneities in the estimated intensity.
    \begin{figure}[h!]
      \centering
      \includegraphics[scale = 0.3]{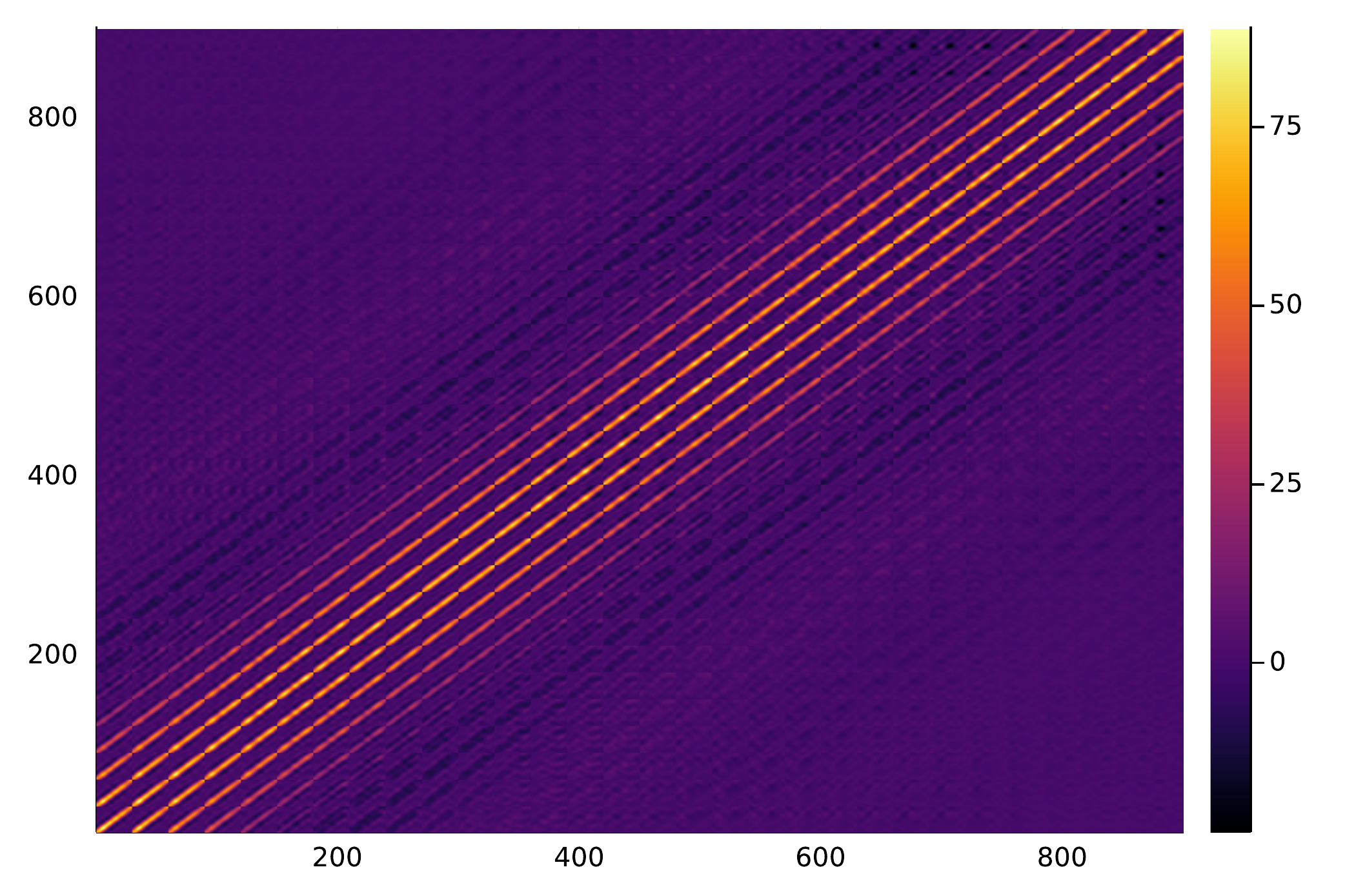}
      \hfill
      \includegraphics[scale = 0.3]{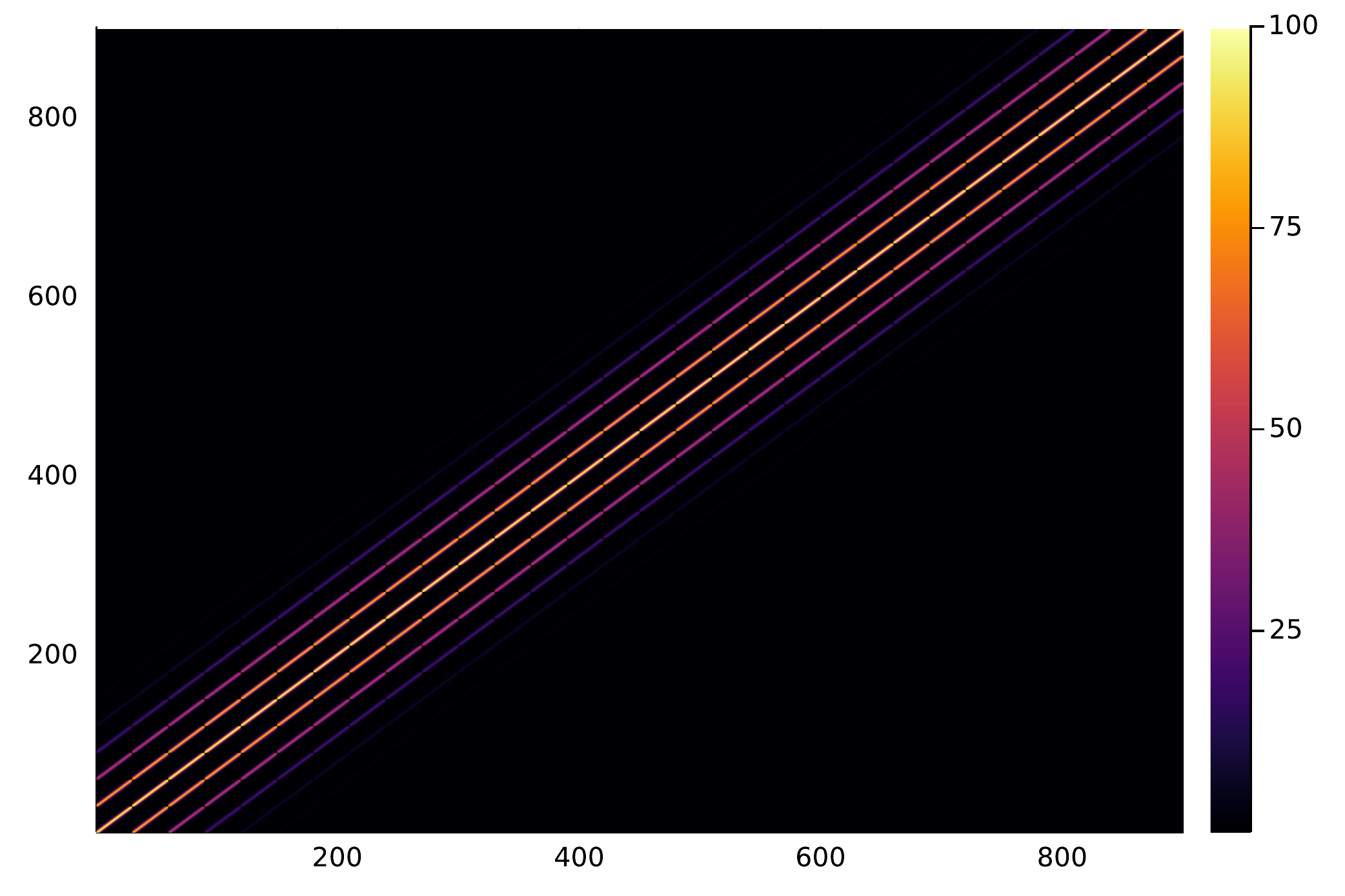}
      \includegraphics[scale = 0.3]{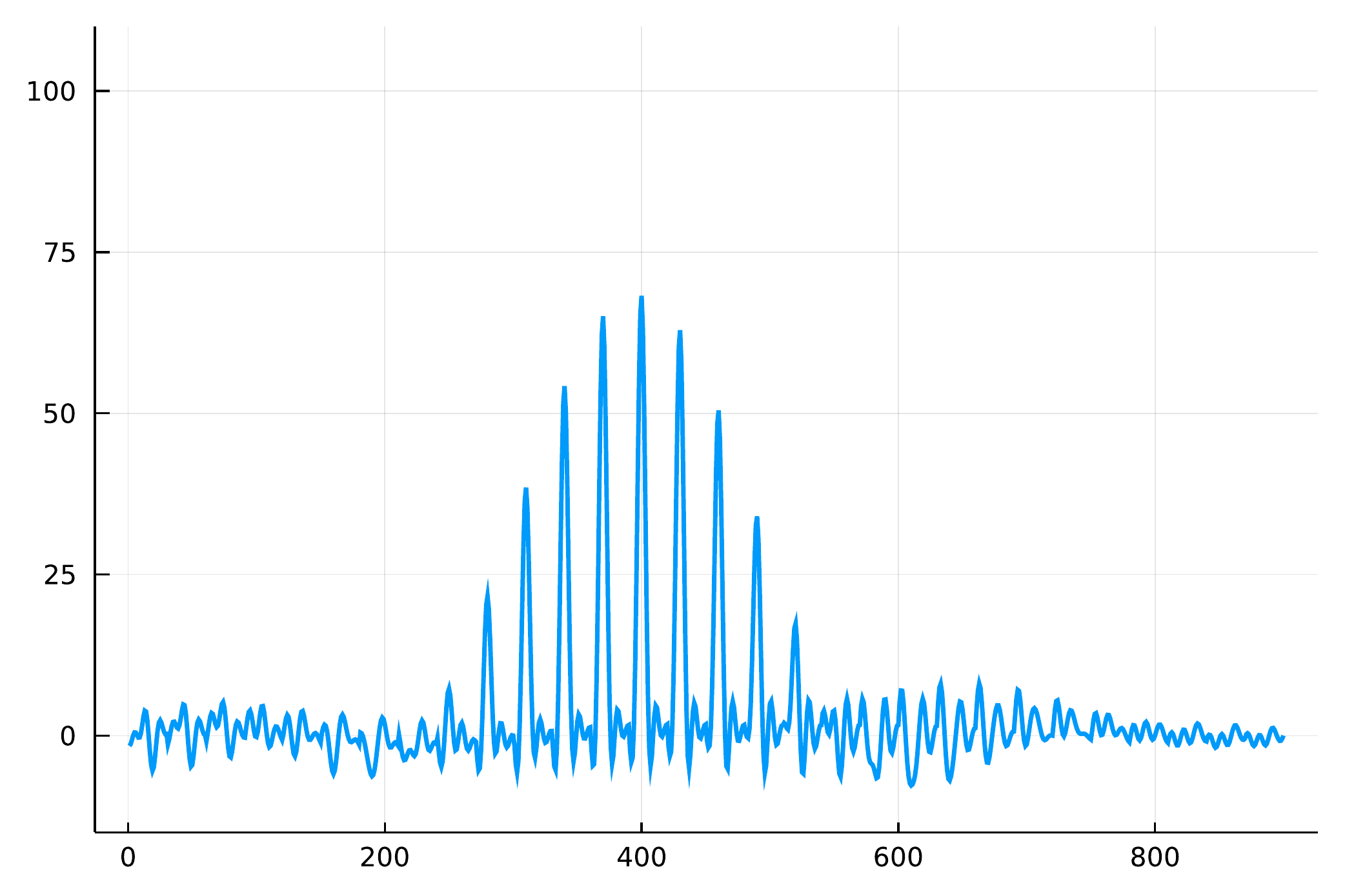}
      \hfill
      \includegraphics[scale = 0.3]{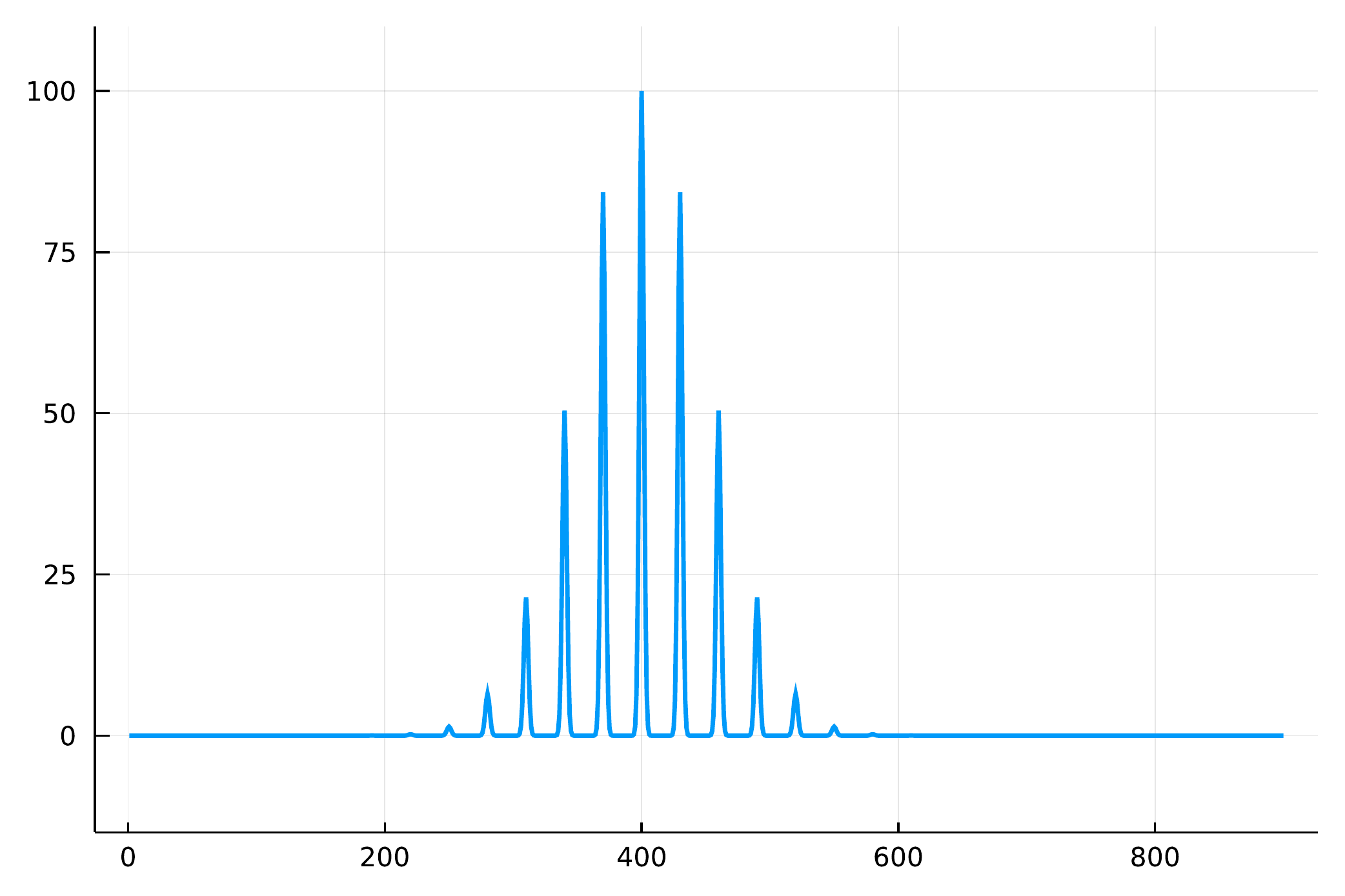}
      \caption{Correlation kernel estimation without boundary effect. We display a Gram matrix of the estimated correlation kernel \MF{$[\hat{\ksf}(x,x')]_{x,x'\in\text{grid}}$} (left column) and ground truth correlation kernel \MF{$[{\ksf}(x,x')]_{x,x'\in\text{grid}}$} (right column) on a $30\times 30$ grid within $[0.2,0.8]^2$ for the example of Figure~\ref{supp:fig:s=10Intensity} with $s=10$, $\sigma = 0.1$ and $\lambda = 10^{-4}$. The first row is a heatmap of the Gram matrices, while the second row is a one-dimensional slice of the above Gram matrices at index $400$. \MF{This second row of plots allows to visually compare the bell shape of the approximate and exact correlation kernels. The apparent discontinuities in the Gaussian kernel shape are an artifact due to the manner the grid points are indexed.} Notice that the correlation kernels are evaluated on a smaller domain within $[0,1]^2$ in order to remove boundary effects.\label{supp:fig:s=10GramMatrix}}
    \end{figure}
    \paragraph{Correlation structure estimation.} The general form of the correlation kernel is also important.  In order to visualize the shape of the correlation kernel $\hat{\ksf}(x,y)$, we display in Figure~\ref{supp:fig:s=10GramMatrix} the Gram matrices of the estimated \MF{$[\hat{\ksf}(x,x')]_{x,x'\in\text{grid}}$} and ground truth correlation kernels \MF{$[\ksf(x,x')]_{x,x'\in\text{grid}}$} on a square grid, for the same parameters as in Figure~\ref{supp:fig:s=10Intensity} (RHS). After removing the boundary effects, we observe that the estimated correlation kernel shape closely resembles the ground truth although the decay of the estimated kernel seems to be a bit slower. Moreover, we observe some `noise' in the tail of the estimated kernel.  Again, the intensity of the estimated process is also a bit underestimated.

    \MF{In the context of point processes, it is common to compute summary statistics from samples to `understand' the correlation structure of a stationary process. It is strictly speaking not possible to calculate e.g.\ Ripley’s K function (see~\citet{Statspats}) since our estimated correlation kernel is not stationary, that is, there exits no function $t(\cdot)$ such that $\hat{\ksf}(x,y) = t(x-y)$.}
    \subsection{Convergence of the regularized Picard iteration}
    In particular, we illustrate the convergence of the regularized Picard iteration to the exact solution given in Proposition~\ref{prop:exact}. In Figure~\ref{supp:fig:convergence}, we solve the problem~\eqref{eq:MLE_Problem_Penalized} in the case $\mathcal I=\mathcal C$ with $s=1$ where $\mathcal{C}$ is the unique DPP sample. For simplicity, we select the DPP sample used in Figure~\ref{fig:Intensity1Sample} ($\rho = 100$, bottom row).
    \begin{figure}[h!]
      \centering
      \includegraphics[scale = 0.3]{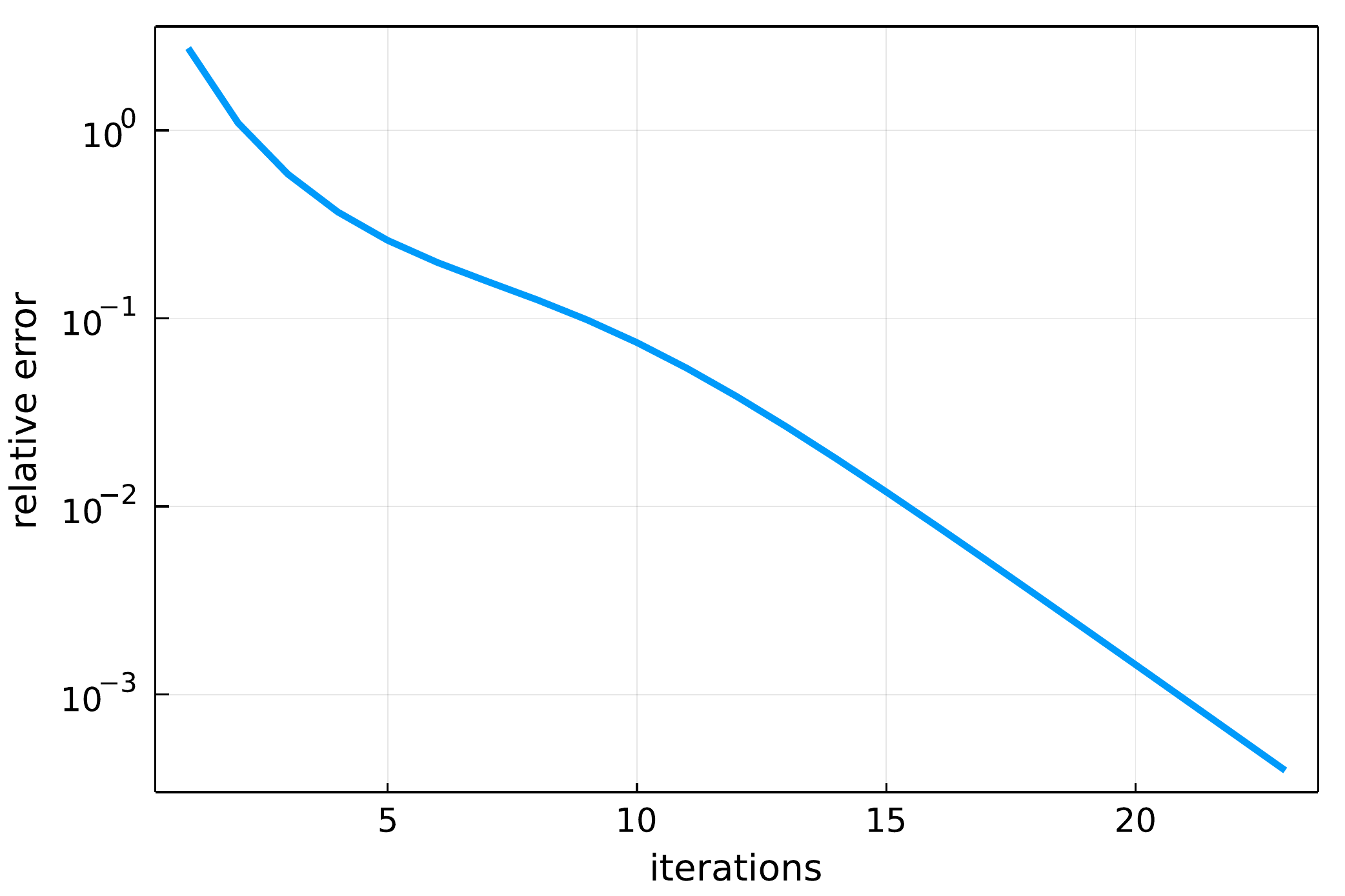}
      \hfill
      \includegraphics[scale = 0.3]{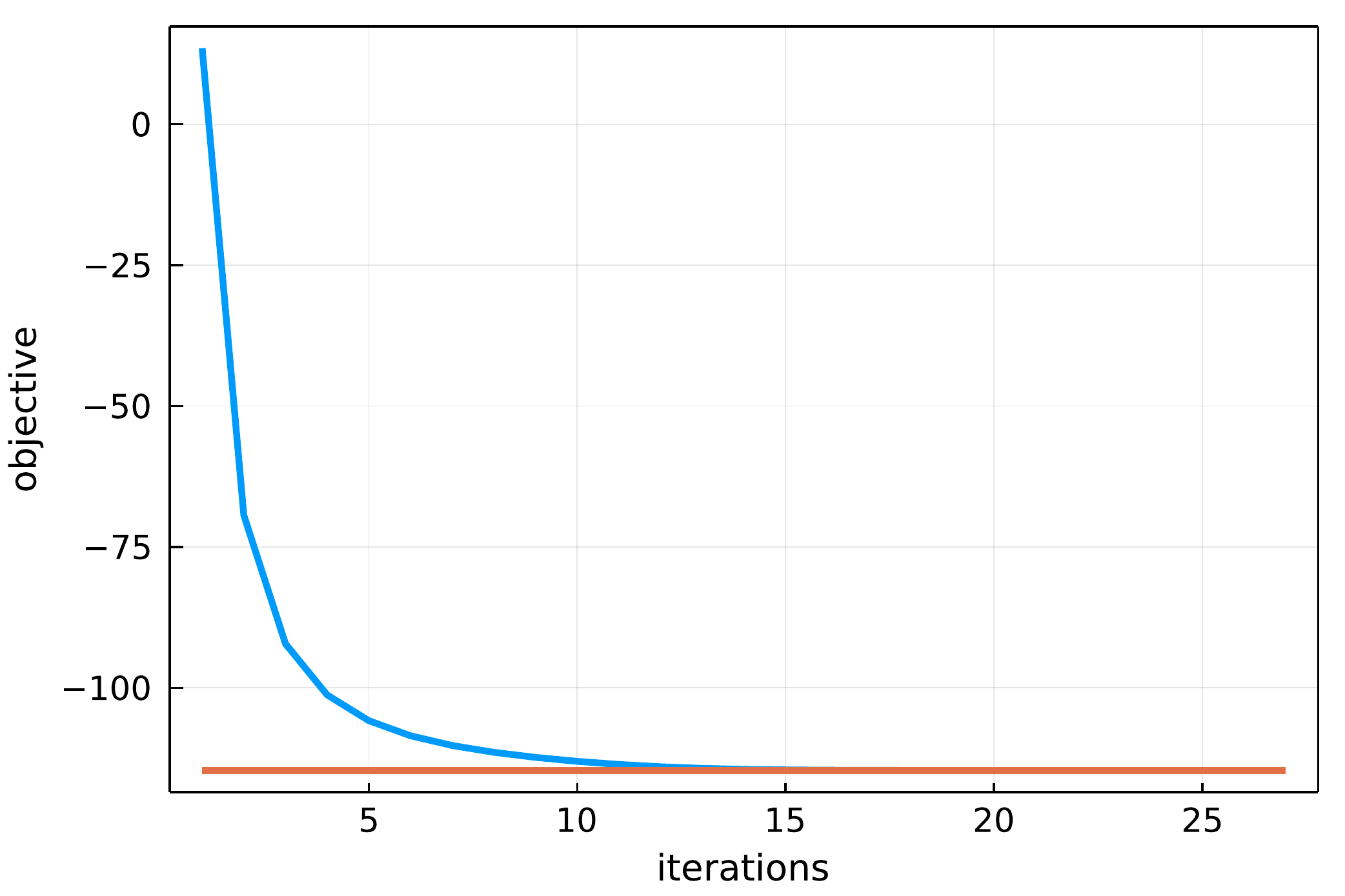}
      \caption{Convergence towards the exact solution for the example of Figure~\ref{fig:Intensity1Sample} (bottom row) with the parameters $\sigma = 0.1$ and $\lambda= 0.1$. Left: Relative error  with the exact solution in Frobenius norm $\|\matB-\matB_{\rm exact}\|_{F}/\|\matB_{\rm exact}\|_{F}$ w.r.t.\ the iteration number. Right: Objective value (blue line) and optimal objective (red line) vs iteration number. The stopping criterion is here $\mathsf{tol} = 10^{-7}$. \label{supp:fig:convergence}}
    \end{figure}
    This illustrates that the regularized Picard iteration indeed converges to the unique solution in this special case.

    \subsection{Complexity and computing ressources}
    \paragraph{Complexity.}
    The space complexity of our method is dominated by the space complexity of storing the kernel matrix $\matK$ in Algorithm~\ref{alg:EstimateL}, namely $O(m^2)$ where we recall that $m= |\mathcal{Z}|$ with $\mathcal{Z} \triangleq  \cup_{\ell = 1 }^{s} \calC_\ell \cup \calI$. The time complexity of one iteration of ~\eqref{eq:regPicard_with_B} is dominated by the matrix square root, which is similar to the eigendecomposition, i.e., $O(m^3)$. \MF{The time complexity of Algorithm~\ref{alg:EstimateK} is dominated by the Cholesky decomposition and the linear system solution, i.e.,  $O(m^3)$.}

    \paragraph{Computing ressources.} A typical computation time is $65$ minutes to solve the example of Figure~\ref{fig:Intensity1Sample} (bottom row) on $8$ virtual cores of a server with two $18$ core Intel Xeon E5-2695 v4s ($2.1$ Ghz). The computational bottleneck is the regularized Picard iteration.
\FloatBarrier
\bibliography{References,stats}

\end{document}

%% file: packages.tex
\usepackage{natbib}
\setcitestyle{authoryear,open={[},close={]}}

\usepackage{nameref,zref-xr}
\zxrsetup{toltxlabel}

\usepackage[final]{neurips_2021}

\usepackage[utf8]{inputenc} 
\usepackage[T1]{fontenc}    
\usepackage[colorlinks=True,citecolor=blue]{hyperref}       
\usepackage{url}            
\usepackage{booktabs}       
\usepackage{amsfonts}       
\usepackage{nicefrac}       
\usepackage{microtype}      
\usepackage{xcolor}         

\usepackage{placeins}         

\usepackage[]{amsmath,amssymb,epsfig}
\usepackage{amsthm}
\usepackage{amsmath}
\usepackage{amssymb}
\usepackage{graphicx}
\usepackage{epstopdf}
\usepackage{comment}
\usepackage{array}
\usepackage{url}
\usepackage{bm}
\usepackage{dsfont}
\usepackage{mathrsfs}

\usepackage{lscape}

\usepackage{algorithm}
\usepackage{algpseudocode}

\usepackage{setspace}
\usepackage{multicol}
\usepackage{multirow}
\usepackage{color}
\usepackage{colortbl}
\usepackage{xcolor}
\usepackage{mathtools}
\usepackage{enumerate}
\usepackage{xifthen}
\usepackage{graphicx}

\usepackage{wrapfig}

%% file: commands.tex
\DeclareMathOperator{\Tr}{Tr}

\DeclareMathOperator{\Span}{\mathrm{span}}

\newtheorem{theorem}{Theorem}
\newtheorem{lemma}[theorem]{Lemma}

\newtheorem{proposition}[theorem]{Proposition}

\newtheorem{corollary}[theorem]{Corollary}

\newcommand{\calH}{\mathcal{H}}

\newcommand{\calI}{\mathcal{I}}
\newcommand{\calC}{\mathcal{C}}
\newcommand{\calD}{\mathcal{D}}

\newcommand{\calR}{d}
\newcommand{\calX}{\mathcal{X}}
\newcommand{\calY}{\mathcal{Y}}

\newcommand{\matK}{\mathbf{K}}
\newcommand{\matC}{\mathbf{C}}
\newcommand{\matLambda}{\mathbf{\Lambda}}
\newcommand{\matR}{\mathbf{R}}
\newcommand{\matB}{\mathbf{B}}
\newcommand{\matX}{\mathbf{X}}
\newcommand{\matU}{\mathbf{U}}
\newcommand{\matSigma}{\mathbf{\Sigma}}
\newcommand{\matDelta}{\mathbf{\Delta}}
\newcommand{\matH}{\mathbf{H}}
\newcommand{\matY}{\mathbf{Y}}
\newcommand{\matN}{\mathbf{N}}
\newcommand{\matM}{\mathbf{M}}
\newcommand{\matI}{\mathbf{I}}
\newcommand{\matQ}{\mathbf{Q}}
\newcommand{\matOmega}{\bm{\Omega}}

\newcommand{\matF}{\mathbf{F}}

\newcommand{\matv}{\mathbf{v}}
\newcommand{\matk}{\mathbf{k}}
\newcommand{\matPhi}{\bm{\Phi}}

\newcommand{\rmd}{\mathrm{d}}

\newcommand{\E}{\mathbb{E}}
\newcommand{\R}{\mathbb{R}}

\newcommand{\I}{\mathbb{I}}

\newcommand{\kRKHS}{k_{\calH}}

\newcommand{\opA}{A}
\newcommand{\opL}{A}

\newcommand{\opS}{S}
\newcommand{\opV}{V}
\newcommand{\bbC}{C}

\newcommand{\Lsf}{\mathsf{A}}
\newcommand{\Kbb}{\mathsf{K}}
\newcommand{\Lk}{\mathsf{T}_{\kRKHS}}

\newcommand{\asf}{\mathsf{a}}
\newcommand{\ksf}{\mathsf{k}}

\newcommand{\Sb}{\mathcal{S}_+}

\newcommand{\MF}[1]{\textcolor{black}{#1}}